\documentclass[runningheads, envcountsect]{llncs}
\usepackage[T1]{fontenc}
\usepackage{graphicx}
\usepackage{amsmath}
\usepackage{amssymb}
\usepackage{mathtools}

\usepackage[utf8]{inputenc} 
\usepackage{amsfonts}       
\usepackage{nicefrac}       
\usepackage{comment}
\usepackage{mymacros}
\usepackage{tikz}
\usepackage{pgfplots}
\usepackage[switch]{lineno} 
\usepackage{multirow}
\usepackage{twoopt}
\usepackage{ifsym}
\usepackage{mathrsfs}  
\usetikzlibrary{positioning,angles,quotes}
\usepackage[capitalize,noabbrev]{cleveref}

\begin{document}
\title{Stepsize Learning for Policy Gradient Methods in Contextual Markov Decision Processes}
\titlerunning{Stepsize Learning for Policy Gradient Methods}
%
\author{Luca Sabbioni\inst{*} \and Francesco Corda \and Marcello Restelli}
\authorrunning{L. Sabbioni et al.}
%
 \institute{Politecnico di Milano, Milan, Italy 
     \\ \email{\{luca.sabbioni, francesco.corda, marcello.restelli\}@polimi.it}
     }

\maketitle              
\begin{abstract}
Policy-based algorithms are among the most widely adopted techniques in model-free RL, thanks to their strong theoretical groundings and good properties in continuous action spaces. Unfortunately, these methods require precise and problem-specific hyperparameter tuning to achieve good performance, and tend to struggle when asked to accomplish a series of heterogeneous tasks. In particular, the selection of the step size has a crucial impact on their ability to learn a highly performing policy, affecting the speed and the stability of the training process, and often being the main culprit for poor results. In this paper, we tackle these issues with a Meta Reinforcement Learning approach, by introducing a new formulation, known as meta-MDP, that can be used to solve any hyperparameter selection problem in RL with contextual processes. After providing a theoretical Lipschitz bound to the difference of performance in different tasks, we adopt the proposed framework to train a batch RL algorithm to dynamically recommend the most adequate step size for different policies and tasks. In conclusion, we present an experimental campaign to show the advantages of selecting an adaptive learning rate in heterogeneous environments.
\end{abstract}

\section{Introduction}
Reinforcement Learning (RL, \cite{Suttonbook:1998}) is a field of Machine Learning aimed at building agents capable of learning a behavior that maximizes the amount of reward collected while interacting with an environment. Typically, this interaction is modeled as a Markov Decision Process (MDP, \cite{puterman2014markov}), where all trajectories share the same transition probability and reward function. Nevertheless, in many real-world scenarios, there may be exogenous variables that can affect the whole dynamics; one might think for example of a car race, where the road temperature or the tire choice may require different strategies.
One of the most successful streams of model-free RL applications adopts policy-based algorithms, which provide solid theoretical groundings and good empirical properties in continuous-action spaces.
Unfortunately, these methods require precise and problem-specific hyperparameter tuning to achieve good performance, causing them to struggle when applied to a series of heterogeneous tasks. The fundamental parameter to tune is the step size, which has a crucial impact on the ability to learn a performing policy, affecting the speed and the stability of the training process, and often being the main culprit for poor results.
Similarly, widely used optimizers (e.g. Adam \cite{kingma2014adam} and RMSProp \cite{tieleman2017divide}) and learning rate schedules have a narrow window of effective hyperparameters \cite{henderson2018did}.
In this work, we consider the specific problem of learning how to dynamically select the best step size for each policy in case the MDP process might differ due to exogenous variables, here denoted as ``tasks'' or ``contexts''. This framework is accurately described by the definition of a Contextual Markov Decision Process (CMDP) introduced in \cite{hallak2015contextual} (Section \ref{sec:pre}). 

Our first original contribution is the formalization of the Meta-RL problem, which we denoted as meta-MDP (Section \ref{sec:metaMDP}). This general framework allows to solve a set of RL tasks, grouped as a CMDP. We discuss the main elements of the model, such as the objective function, which is performance learning, and the meta action, consisting of the hyperparameter selection for a policy update. 
In this framework, we then add an assumption of Lipschitz continuity of the meta-MDPs, in which trajectories sampled from similar contexts are similar. This is a reasonable assumption for real-world problems, where a small change in the settings slightly changes the effects on the dynamics of the environment.
Under such conditions, it is possible to derive some guarantees on the Lipschitz continuity of the expected return and of its gradient (Section \ref{sec:lip_metaMDP}).
This is relevant, as it gives insight into the generalization capabilities of meta-RL approaches, where the performance of policies selected by observing tasks in training can be bounded for test tasks.
Subsequently, we propose in Section \ref{sec:metaFQI} to learn the step size of Policy Gradient methods in a meta-MDP. The idea of the approach is to apply a batch mode, value-based algorithm, known as Fitted Q-Iteration (FQI), to derive an estimate of the (meta) action-value function, based on the meta-features observed and of the hyperparameter selected. This approximation is used to dynamically recommend the most appropriate step size in the current scenario. The learning procedure is based on a regression through ExtraTrees \cite{geurts2006extremely}, which shows low sensitivity to the choice of its own parameters.
In conclusion, we evaluate our approach in various simulated environments shown in Section \ref{sec:experiments}, highlighting its strengths and current limitations.

\section{Related work}
The importance of hyperparameter tuning is widely known in the general Machine Learning field, because it can significantly improve the performance of a model \cite{henderson2018did,weerts2020importance,jastrzebski2020break}. Therefore, Hyperparameter Optimization (HO) is a paramount component of Automated Machine Learning (AutoML, \cite{hutter2019automated}) with a rich stream of research literature
\cite{parker2022automated}.\newline

The tuning process is usually approached by practitioners as a black-box approach: the most common methods are grid search or random search \cite{bergstra2012random}. More advanced methods are obtained by relying on sequential model-based Bayesian optimization \cite{hutter2011sequential,feurer2015initializing,snoek2012practical}, where a probabilistic model is trained to fit the underlying fitness function of the main learning algorithm. In some recent works \cite{eiben2007reinforcement,sehgal2019deep}, Genetic Algorithms are employed to automatically learn the most performing parameters on RL applications. The main limitation in this kind of approach consists of the need for complete learning instances to evaluate each hyperparameter, which is kept fixed throughout the whole process
A completely different solution consists in training an outer network, typically an RNN \cite{meier2018online,ravi2017opt,andrychowicz2016learning,im2021online} since it is often possible to compute the gradient of the objective function w.r.t. the hyperparameters through implicit differentiation, as shown in \cite{maclaurin2015gradient,lorraine2020optimizing}. These methods are often referred to as \textit{bilevel optimization} procedures, where the \textit{outer} loop updates the hyperparameters on a validation set, and the \textit{inner} one is used for training the models with a specific hyperparameter set.

Recent independent papers introduced the formal paradigm of Dynamic Algorithm Configuration and HO as a Sequential Decision Process \cite{adriaensen2022automated,biedenkapp2020dynamic,jomaa2019hyp}, albeit many other works developed solutions in this direction, employing RL-based methods \cite{zhu2019gradient,li2016learning,xu2017reinforcement,zoph2017neural} or contextual bandits \cite{li2017hyperband}.
However, these works are rarely adopted in RL, as they become computationally intractable and sample inefficient. Furthermore, gradient-based methods \cite{xu2018meta} compute the gradient of the return function with respect to the hyperparameters: they rely on a strong assumption that the update function must be differentiable and the gradient must be computed on the whole chain of training updates. In addition, these approaches are typically online, with limited exploration (as discussed in \cite{biedenkapp2020dynamic}), or make use of gradient-based meta-algorithms, where the high level of sensitivity to new meta-hyperparameters makes the problem even more challenging, as the models may be harder to train and require more data.
Within the specific task of learning rate tuning in a policy-gradient framework, \cite{Supratik2019Fast} proposed a sample efficient algorithm to learn a hyperparameter schedule employing a Weighted Importance Sampling approach, while \cite{paine2020hyperparameter} deals with the offline hyperparameter selection for offline RL. In these proposed approaches, HO is meant to optimize the objective function in the next step, similar to a bandit problem, which favors convergence to local optima. In order to optimize over a longer horizon, \cite{shala2020learning} adopts an RL approach to select the learning rate through Guided Policy Search. \newline

The concept of rapid adaptation to unseen tasks is usually denoted as meta-learning \cite{schmidhuber1987evolutionary} and has recently emerged as a fertile and promising research field, especially with regard to gradient-based techniques. One of the cornerstones in this area is MAML \cite{finn2017modelagnostic}, which learns a model initialization for fast adaptation and has been a starting point for several subsequent works \cite{nichol2018reptile,park2019metacurv}. PEARL \cite{rakelly2019efficient} decouples the problem of making an inference on the probabilistic context and solving it by conditioning the policy in meta Reinforcement Learning problems. However, all these works heavily rely on choosing (multiple) learning rates.

\section{Preliminaries}\label{sec:pre}
\sloppy A discrete-time Markov Decision Process (MDP) is defined as a tuple $\langle\Sspace,\Aspace, \mathcal{P}, \mathcal{R}, \gamma, \mu\rangle$, where $\Sspace$ is the (continuous) state space, $\Aspace$ the (continuous) action space, $\mathcal{P}(\cdot|s,a)$ is the Markovian transition, which assigns to each state-action pair $(s,a)$ the probability of reaching the next state $s'$, $\mathcal{R}$ is the reward function, bounded by hypothesis, i.e. $\sup_{s\in\Sspace, a\in\Aspace}|\mathcal{R}(s,a)|\leq\Rmax$ . Finally, $\gamma\in[0,1]$ is the discount factor, and $\mu$ is the initial state distribution. The policy of an agent, denoted as $\pi(\cdot|s)$, assigns to each state $s$ the density distribution over the action space $\Aspace$.

A trajectory $\tau\coloneqq (s_0, a_0, s_1, a_1, s_2, a_2, ..., a_{H-1}, s_H)$ is a sequence of state-action pairs, where $H$ is the horizon, which may be infinite. The return of a trajectory $\tau$ is defined as the discounted sum of the rewards collected: $G_\tau = \sum_{t=0}^H \gamma^t \mathcal{R}(s_t,a_t).$
Consequently, it is possible to define the expected return $j_\pi$ as the expected performance under policy $\pi$.
Similarly, we can define, for each state $s\in\Sspace$ and action $a\in\Aspace$, the (action)-value functions as:
\begin{align*}
    \small{Q_\pi(s,a) \coloneqq \EV_{\substack{s_{t+1}\sim \mathcal{P}(\cdot|s_{t},a_{t})\\a_{t+1}\sim\pi(\cdot|s_{t+1})}}}&\small{\left[\sum_{t=0}^\infty \gamma^t \mathcal{R}(s_t,a_t)|s,a\right]} \\ V_\pi(s)\coloneqq  \EV_{a\sim\pi(\cdot|s)}&[Q_\pi(s,a)].
\end{align*}
For the rest of the paper, we consider parametric policies, where the policy $\pi_{\vtheta}$ is parameterized by a vector $\vtheta\in \Theta \subseteq \mathbb{R}^m$. In this case, the goal is to find the optimal parametric policy that maximizes the performance, i.e. $\vtheta^* = \arg\max_{\vtheta\in\Theta}j(\vtheta)$\footnote{For the sake of brevity, when a variable depends on the policy $\pi_{\vtheta}$, in the superscript only $\vtheta$ is shown.}.
Policy-based algorithms adopt a gradient-ascent approach: the Policy Gradient Theorem (PGT)~\cite{Suttonbook:1998} states that, for a given policy $\pi_{\vtheta},\  \vtheta\in\Theta$:
\begin{equation}\label{eq:PGT}
    \nabla_{\vtheta} j(\vtheta) = \EV_{\substack{s\sim \delta_{\mu}^{\vtheta}\\a\sim\pi_{\vtheta}(\cdot|s)}}\bigg[\nabla_{\vtheta}\log\pi_{\vtheta}(a|s)Q_\pi(s,a)\bigg],
\end{equation}
where $\delta_{\mu}^{\vtheta}$ is the state occupancy measure induced by the policy, in such a way that $\delta_{\mu}^{\vtheta}(s)\coloneqq (1-\gamma)\int_\Sspace \mu(s_0)\sum_{t=0}^T\gamma^tp_{\vtheta}(s_0 \xrightarrow{t} s)\de s_0$, with $p_{\vtheta}(s_0 \xrightarrow{t} s)$ being the probability of reaching state $s$ from $s_0$ in $t$ steps following $\pi_{\vtheta}$.
In practice, the gradient in Equation \ref{eq:PGT} can be computed only through an estimator $\widehat{\nabla}_N j_{\vtheta}$, such as PGT \cite{sutton_pgt_2000}, that requires sampling a  batch of trajectories $\{\tau_i\}_{i=1}^N$.
A large family of algorithms is based on the Policy Optimization through Gradient Ascent, eventually with the inclusion of other methods, such as Trust Regions and constraints over the Kullback-Leibler divergence of the policies between consecutive iterations \cite{schulman2015trust,schulman2017proximal}. An important variation on the approach consists in following the steepest ascent direction using the Natural Policy Gradient \cite{kakade2001natural}, which includes information regarding the curvature of the return manifold over the policy space in the form of the Fisher Information Matrix $F(\vtheta)=\EV[\nabla_{\vtheta}\log\pi_{\vtheta}\nabla^\top_{\vtheta}\log\pi_{\vtheta}]$; its inverse is then multiplied by the gradient to obtain the natural gradient $g(\vtheta)\coloneqq F(\vtheta)^{-1}\nabla_{\vtheta}j(\vtheta)$, independent of the policy parameterization. A common approach to avoid long computational times for large policy spaces is to directly provide an estimate of the natural gradient $\widehat{g}_N(\vtheta)$ by using the same batch of trajectories adopted for the gradient estimation, and through the iteration of $k$ steps of conjugate gradient methods with the application of the Fisher-vector products \cite{schulman2015trust}. 

\vspace{0.15cm}
\textit{Lipschitz MDP}
This subsection introduces the concepts of Lipschitz Continuity (LC) and Lipschitz MDP. The notation is taken from \cite{pirotta2015policy}. 
Let $(\mathcal{X}, d_{\mathcal{X}})$ and $(\mathcal{Y}, d_{\mathcal{Y}})$ be two metric spaces; a function $f : \mathcal{X} \rightarrow \mathcal{Y}$ is called $L_f$-Lipschitz continuous ($L_f$-LC), with $L_f \ge 0$, if $d_{\mathcal{Y}}(f(x),f(x')) \le L_f d_{\mathcal{X}} (x,x') \forall x,x' \in \mathcal{X}$. Furthermore, we define the Lipschitz semi-norm as $\|f\|_L = \sup_{x,x' \in \mathcal{X} : x\neq x'} \frac{d_{\mathcal{Y}}(f(x),f(x'))}{d_{\mathcal{X}} (x,x')}$. For real functions, the usual metric is the Euclidean distance while, for distributions, a common metric is the Kantorovich, or $L^1$-Wasserstein distance:
$$\mathcal{K}(p,q)\coloneqq \sup_{f: \|f\|_L\le 1}\bigg\{\big{\|}\int_X f d(p-q)\big{\|}\bigg\}$$
~\cite{rachelson2010locality,pirotta2015policy} introduced some notion of smoothness in RL by defining the Lipschitz-MDP and the Lipschitz policies:
\begin{ass}\label{ass:LipMDP}
	Let $\mathcal{M}$ be an MDP. $\mathcal{M}$ is called $(L_{P},L_r)$-LC if for all $(s,a),(\overline{s},\overline{a}) \in \SAs$:
	\begin{align*}
		& \mathcal{K} \left( P (\cdot|s,a), P(\cdot|\overline{s},\overline{a}) \right) \le L_{P} \,  d_{\SAs} \left((s,a), (\overline{s},\overline{a}) \right), \\
		& \left| r(s,a) - r(\overline{s},\overline{a}) \right| \le L_r \,  d_{\SAs} \left((s,a), (\overline{s},\overline{a}) \right).
	\end{align*}
\end{ass}
\begin{ass}\label{ass:Lip_Policy}
	Let $\pi \in \Pi$ be a Markovian stationary policy. $\pi$ is called $L_{\pi}$-LC if for all $s,\overline{s} \in \Ss$:
	\begin{align*}
		& \mathcal{K} \left( \pi (\cdot|s), \pi(\cdot|\overline{s}) \right) \le L_{\pi} \,  d_{\Ss} \left(s,\overline{s} \right),
	\end{align*}
\end{ass}
Since we are dealing with parametric policies, often other useful assumptions rely on the Lipschitz continuity \wrt the policy parameters $\vtheta$ and their gradient.
In \cite{pirotta2015policy}, it is shown that, under these Lipschitz continuity assumptions on the MDP and the policy model, also the expected return, the $Q$-function, and the gradient components are Lipschitz w.r.t. $\vtheta$.\footnote{By assuming that the policy and its gradient is LC w.r.t. $\vtheta$.}

\textbf{Meta Reinforcement Learning.} As the name suggests, meta-learning implies a higher level of abstraction than regular machine learning. In particular, meta reinforcement learning (meta-RL) consists in applying meta-learning techniques to RL tasks. Usually, these tasks are formalized in MDPs by a common set of parameters, known as the \textit{context} $\vomega$.
The natural candidate to represent the set of RL tasks is the Contextual Markov Decision Process (CMDP, \cite{hallak2015contextual}), defined as a tuple $(\Omega, \mathcal{S}, \mathcal{A}, \mathcal{M}(\vomega))$ where $\Omega$ is called the context space, $\mathcal{S}$ and $\mathcal{A}$ are the shared state and action spaces, and $\mathcal{M}$ is the function that maps any context $\vomega \in \Omega$ to an MDP, such that $\mathcal{M}(\vomega)=\langle\mathcal{S}, \mathcal{A}, P_{\vomega}, R_{\vomega}, \gamma_{\vomega}, \mu_{\vomega}\rangle$.
In other words, a CMDP includes in a single entity a group of tasks. In the following, we will assume that $\gamma$ and $\mu$ are shared, too.

\section{Meta-MDP}\label{sec:metaMDP}
We now present the concept of meta-MDP, a framework for solving meta-RL tasks that extends the CMDP definition to include the learning model and the policy parameterization. Similar approaches can be found in \cite{garcia2019metamdp} and in \cite{li2016learning}.
To start, let's consider the various tasks used in a meta-training procedure as a set of MDPs $\left\{\mathcal{M}_{\vomega}\right\}_{\vomega \in \Omega}$, such that each task $\mathcal{M}_{\vomega}$ can be sampled from the distribution $\psi$ defined on the context space $\Omega$. This set can be seen equivalently as a CMDP $\mathscr{M}=\langle\Omega, \Sspace,\Aspace, \mathcal{M}(\vomega)\rangle$, where $\mathcal{M}(\vomega)=\mathcal{M}_{\vomega}$.  Similarly, we define a distribution $\rho$ over the policy space $\mathbf{\Theta}$, so that at each iteration in an MDP $\mathcal{M}_{\vomega}$, the policy parameters $\vtheta_0$ are initialized to a value sampled from $\rho$. In our case, we assume to be able to represent the task by the parameterized context itself $\vomega$.
\begin{defi}\label{def:metaMDP}
A meta-MDP is defined as a tuple $\langle\mathcal{X}, \mathcal{H}, \mathcal{L}, \widetilde{\gamma},(\mathscr{M}, \psi),(\boldsymbol{\Theta}, \rho), f\rangle$,
where:
\begin{itemize}\setlength\itemsep{.1mm}
    \item $\mathcal{X}$ and $\mathcal{H}$ are respectively the meta observation space and the learning action space;
    \item $\mathcal{L}:\vTheta \times \Omega\times\mathcal{H}\rightarrow \mathbb{R}$ is the meta reward function;
    \item $\widetilde{\gamma}$ is the meta-discount factor;
    \item $(\mathscr{M}, \psi)$ and $(\boldsymbol{\Theta}, \rho)$ contain respectively a CMDP $\mathscr{M}$ with distribution over tasks $\psi$, and the policy space $\boldsymbol{\Theta}$, with initial distribution $\rho$;
    \item $f$ is the update rule of the learning model chosen.
\end{itemize}
\end{defi}
In particular, a meta-MDP attempts to enclose the general elements needed to learn an RL task into a model with properties similar to a classic MDP. 
The meta observation space $\mathcal{X}$ of a meta-MDP can be considered as the generalization of the observation space in classic Partially-Observable MDPs (POMDP) \cite{aastrom1965optimal}, and it is meant to include information regarding the current condition of the learning process, and it is (eventually implicitly) dependent on $\vtheta$ and on the context $\vomega$. 

Each action $h_k \in \mathcal{H}$ performed on the meta-MDP with policy parametrization $\vtheta_k$ at the $k$-th step, determines a specific hyperparameter that regulates the stochastic update rule $f$, i.e., $\boldsymbol{\theta}_{k+1}=f\left(\boldsymbol{\theta}_{k}, h_{k}, \tau_k\right)$, where $\tau_k$ is the current batch of trajectories.
In general, we can consider any update function with a set of tunable hyperparameters; in particular, in this work we focus on (Normalized) Natural Gradient Ascent (NGA), in which the action $h$ determines the step size, and the update rule takes the form $f(\boldsymbol{\theta}, h)=\boldsymbol{\theta}+h \frac{\widehat{g}_N (\boldsymbol{\theta},\vomega)}{\|\widehat{g}_N (\boldsymbol{\theta},\vomega)\|_2}$, where $\widehat{g}_N(\vtheta,\vomega)$ is the natural gradient of a policy $\vtheta$ estimated on $N$ episodes through the task $\mathcal{M}_{\vomega}$.

As in a standard RL problem, the training of a meta-MDP is accomplished by optimizing a reward function. Meta-Learning has the main goal of learning to learn: as a consequence, we want to consider performance improvement as our reward. To accelerate the learning over the current MDP $\mathcal{M}_{\vomega}$, this function should reflect variations between the returns obtained in different learning steps. To accomplish this, we define $\mathcal{L}(\boldsymbol{\theta},\vomega, h)$ as a function of the current policy parameters $\vtheta$ and of the meta-action $h$ once the context $\vomega$ is fixed:
\begin{equation*}
    \mathcal{L}(\boldsymbol{\theta}, \vomega, h):=j_{\vomega}(f(\boldsymbol{\theta}, h))-j_{\vomega}(\boldsymbol{\theta});
\end{equation*}
where $j_{\vomega}(\boldsymbol{\theta})$ and $j_{\vomega}(f(\boldsymbol{\theta}, h))$ are respectively the expected returns in the task $\mathcal{M}_{\vomega}$ before and after one update step according to the function $f$, estimated through a batch of sampled trajectories.
In the particular case of NGA, the function takes the following form:
\begin{equation*}
    \mathcal{L}(\boldsymbol{\theta},\vomega, h)=j_{\vomega}\left(\boldsymbol{\theta}+h \frac{\widehat{g}_N (\boldsymbol{\theta},\vomega)}{\|\widehat{g}_N (\boldsymbol{\theta},\vomega)\|_2}\right)-j_{\vomega}(\boldsymbol{\theta}).
\end{equation*}
Unlike a standard MDP, a meta-MDP does not include a Markovian transition model that regulates its dynamics: given $x_k\in\mathcal{X}$, the stochastic transition to the next meta-state $x_{k+1}$ is induced by the distribution of the trajectories induced by the pair $(\theta_k, \mathcal{M}_{\vomega})$ and on the update rule $f$. The initial state hence implicitly depends on $\psi$ and $\rho$, and the transition to the next state is still Markovian, as it is independent of the previous states observed (once $x_k$ is known).
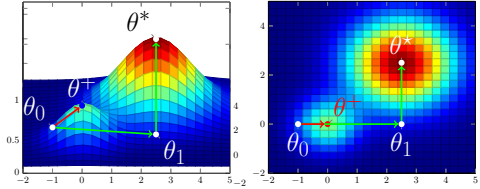
\begin{figure}[t]
    \centering
    \def\centerx{2.5}
\def\centery{2.5}
\begin{tikzpicture}[scale=0.4]
\begin{axis}[view={0}{45}, name=plot1, pin distance=.6mm]
\addplot3[surf,domain=-2:5,domain y=-2:5,  colormap/jet]
{exp(-( (x-\centerx)^2 + (y-\centery)^2)/3 )+0.5*exp(-( (x)^2 + (y)^2)/1 )};
\node[circle,minimum size=0.2cm,inner sep=1pt,fill=white,pin=111:\Huge{$\theta^*$}]
at (axis cs:\centerx,\centery,1) (D) {};
\node[circle,minimum size=0.2cm,inner sep=1pt,fill=white,pin=130:\Huge{\textcolor{white}{$\theta_0$}}]
at (axis cs:-1,0,0.197) (A) {};
\node[circle,minimum size=0.2cm,inner sep=1pt,fill=blue,pin=91.5:\Huge{\textcolor{white}{$\theta^+$}}]
at (axis cs:0,0,0.5) (B) {};
\node[circle,minimum size=0.2cm,inner sep=1pt,fill=white,pin=320:\Huge{\textcolor{white}{\textbf{$\theta_1$}}}]
at (axis cs:2.5,0,0.1) (C) {};
\draw[->,draw=red, ultra thick] (A) -- (B);
\draw[->,draw=green,ultra thick] (A) -- (C);
\draw[->,draw=green,ultra thick] (C) -- (D);
\end{axis}
\begin{axis}[view={0}{90},pin distance=0.6mm,  name=plot2,
    at=(plot1.right of south east), anchor=left of south west]
    \addplot3[surf,domain=-2:5,domain y=-2:5, colormap/jet]
    {exp(-( (x-\centerx)^2 + (y-\centery)^2)/3 )+0.5*exp(-( (x)^2 + (y)^2)/1 )};
    \node[circle,minimum size=0.2cm,inner sep=1pt,fill=white,pin=90:\Huge{\textcolor{white}{$\theta^\star$}}]
    at (axis cs:\centerx,\centery,1) (D) {};
    \node[circle,minimum size=0.2cm, inner sep=1pt,fill=white,pin=270:\Huge{\textcolor{white}{$\theta_0$}}]
    at (axis cs:-1,0,0.197) (A) {};
    \node[circle,minimum size=0.2cm,inner sep=1pt,fill=red,pin=60:\Huge{\textbf{\textcolor{red}{$\theta^+$}}}]
    at (axis cs:0,0,0.5) (B) {};
    \node[circle,minimum size=0.2cm,inner sep=1pt,fill=white,pin=270:\Huge{\textcolor{white}{$\theta_1$}}]
    at (axis cs:2.5,0,0.1) (C) {};
    \draw[->,draw=green,ultra thick] (A) -- (C);
    \draw[->,draw=green,ultra thick] (C) -- (D);
    \draw[->,draw=red, ultra thick] (A) -- (B);
    \end{axis}
\end{tikzpicture}
\caption{Example of an optimization setting where a Bandit approach would be suboptimal: starting from $\theta_0$, the optimal bandit agent will choose to reach $\theta^{+}$, a local maximum. An RL agent, however, may plan to make a larger step, up to $\theta_1$, to reach the global optimum $\theta^\star$ on the next update.}
\label{fig:example}
\end{figure}

\vspace{0.15cm}
\textit{Discount factor, contextual bandit, and meta-MDP.}
The choice of the meta-discount factor $\widetilde{\gamma}$ is critical: meta-learning is very often considered as paired with \emph{few-shot learning}, where a short horizon is taken into account for the learning process. $\widetilde{\gamma}$, if lower than 1,  explicitly translates into an effective horizon of $\frac{1}{1-\widetilde{\gamma}}$. However, a myopic behavior induced by a low discount factor might lead the meta-agent to prefer actions leading to local optima, while sometimes it might be necessary to take more cautious steps to reach the global optima of the learning process. Setting $\widetilde{\gamma}=0$, the problem degenerates into a contextual bandit, where the goal is to maximize the immediate reward, in a similar fashion as in \cite{Supratik2019Fast}. However, it might be inefficient to directly maximize the immediate reward, as an agent might prefer to choose a different hyperparameter to reach the global optimum, which is possibly unreachable in just one step. 
Figure \ref{fig:example} provides an example in this direction, where a bi-dimensional parametrization is considered: starting from the initial parametrization $\vtheta_0$, the maximization of the immediate return would lead to a local optimum $\vtheta^+$. We want our agent to be able to plan the updates to maximize the final policy's performance: this is the main reason for the design of HO as a sequential decision-making problem.  

\vspace{0.15cm}
\textit{Meta-Space features:} In this subsection, we deal with the choice of the features observed in the meta-observation $x_t$. Some properties are generally desirable for its formulation: first of all, it needs to include policy-specific information, as some form of knowledge about the current policy is necessary to adapt the meta-actions to the current setting of the model. Ideally, we can include all parameters of the current policy $\vtheta_t$, even if this approach might be difficult for large policy spaces. Finding an informative set of meta-features remains an open problem for future research, as recalled in Section \ref{sec:conc}.
Additionally, task-specific features may be informative. The information about the task $\vomega$ is used to achieve an \textit{implicit task-identification}, a necessary step to optimize learning in new tasks, based on similarities to older ones. 
Finally, some relevant information could be included in the (natural) gradient $\widehat{g}_N(\theta_t,\vomega)$: this vector is implicitly dependent on the stochasticity of the inner MDP $\mathcal{M}_{\vomega}$ under policy $\vtheta_t$ according to the batch of trajectories sampled for its estimation.
In our experiments, we will consider the concatenation of all these features $x_t = \langle \vtheta_t, \widehat{g}_N(\theta_t, \vomega), \vomega \rangle$.
From a more technical point of view, a Meta-MDP can be considered as the conditional observation probability of a POMDP, where the true state consists of the pair $(\theta_t, \vomega)$, and the meta-observation $x_t$ relies on a conditional observation probability $\mathcal{O}(\cdot|\theta_t,\vomega)$.

\section{Context Lipschitz Continuity}\label{sec:lip_metaMDP}
We consider a meta-MDP in which all inner tasks satisfy the Lipschitz continuity assumption. Under this condition, we can derive a set of bounds on the approximation errors obtained by the meta-agent when acting on unseen tasks. Among others, we obtain that the expected return $j_{\vomega}(\boldsymbol{\theta})$ and its gradient are LC w.r.t. the context $\vomega$, providing useful theoretical foundations for the meta-RL general framework and inspiring motivation to look for solutions and models capable of generalizing on large task spaces. Let's suppose to be provided with a CMDP $(\Omega, \Sspace, \Aspace, \mathcal{M})$, such that Assumption \ref{ass:LipMDP} is verified $\forall\vomega\in\Omega$, meaning that $\forall\vomega\in\Omega$ the MDP $\mathcal{M}_{\vomega}$ is $(L_P(\vomega)-L_r(\vomega))$-LC.
Let us also assume that the set of MDPs is LC in the context $\vomega$:
\begin{ass}\label{as:lip_metamdp}
Let $\mathscr{M}$ be a CMDP. $\mathscr{M}$ is called $(L_{\omega_P},L_{\omega_r})$-Context Lipschitz Continuous ($(L_{\omega_P},L_{\omega_r})$-CLC) if for all $(s,a),(\overline{s},\overline{a}) \in \SAs$, $\forall \vomega, \widehat{\vomega}\in\Omega$:
\begin{align*}
    \left.\mathcal{K}\left(P_{\vomega}(\cdot \mid s, a), P_{\widehat{\vomega}}(\cdot \mid s, a)\right)\right)  &\leq L_{\omega_{P}} d_{\Omega}(\vomega, \widehat{\vomega}) \\
    \Bigl|R_{\vomega}(s, a)-R_{\widehat{\vomega}}(s, a)\Bigr|  &\leq L_{\omega_{r}} d_{\Omega}(\vomega, \widehat{\vomega}).
\end{align*}
\end{ass}
This means we have some notion of task smoothness: when two MDPs with similar contexts are considered, their transition and reward processes are similar.
These assumptions, along with Assumption \ref{ass:Lip_Policy}, allow us to infer some considerations regarding the Q-value function: 
\begin{restatable}[]{thm}{lipschitzQ}\label{thm:lip_q_omega}
Let $\mathscr{M}$ be a $(L_{\omega_P},L_{\omega_r})$-CLC CMDP for which $\mathcal{M}(\vomega)$ is $\left(L_{P}(\vomega), L_{r}(\vomega)\right)$-LC $\forall \vomega \in \Omega$. Given a $L_{\pi}$-LC policy $\pi$, the action value function $Q_{\omega}^{\pi}(s, a)$ is $L_{\omega_{Q}}$-CLC w.r.t. the context $\vomega$, i.e., $\forall (s,a)\in\Sspace\times\Aspace$:
\begin{equation*}\label{eq:Q_CLC}
    \Bigl|Q_{\vomega}^{\pi}(s, a)-Q_{\widehat{\vomega}}^{\pi}(s, a)\Bigr| \leq L_{\omega_{Q}}(\pi) d_{\Omega}(\vomega, \widehat{\vomega}),
\end{equation*}
where
\begin{equation}\label{eq:L_Qw}
\begin{aligned}
    L_{\omega_{Q}}(\pi)&=\frac{L_{\omega_{r}}+\gamma L_{\omega_{p}} L_{V_{\pi}}(\vomega)}{1-\gamma},\\
    L_{V_{\pi}}(\vomega)&=\frac{L_r(\omega)(1+L_\pi)}{1-\gamma L_P(\vomega)(1+L_\pi)}
    \end{aligned}
\end{equation}
\end{restatable}

As a consequence, the return function $j_{\vomega}(\pi)$ is context-LC: $\left|j_{\vomega}(\pi)-j_{\widehat{\vomega}}(\pi)\right|\leq L_{\omega_{Q}}(\pi) d_{\Omega}(\vomega, \widehat{\vomega}).$
In simpler terms, Theorem \ref{thm:lip_q_omega} exploits the LC property to derive an upper bound on the return distance in different tasks. This result represents an important guarantee on the generalization capabilities of the approach, as it provides a boundary on the error obtained in testing unseen tasks.
A proof for this theorem is provided in the supplementary material, where we also prove that the analytic gradient $\nabla j_{ \vomega}^{\boldsymbol{\theta}}$ is CLC w.r.t. the context, too. In particular, a bound on the distance between the gradients of different tasks ensures regularity in the surface of the return function, which is important as the gradient is included in the meta state to capture information regarding the context space.

\begin{algorithm}[t]
\caption{Meta-MDP Dataset Generation for NGA (trajectory method)}
\label{alg:FQIdata}
	\begin{algorithmic}
		\STATE \textbf{Input:} CMDP $\mathscr{M}$,  task distribution $\psi$, policy space $\vTheta$, initial policy distribution $\rho$, \\ number of meta episodes $K$, learning steps $T$, inner trajectories $N$.
		\STATE \textbf{Initialize:} $\mathcal{F} =\{\}$,
		\FOR{$k=1,\dots,K$}
			\STATE Sample context $\vomega\sim \psi(\Omega)$, initial policy $\vtheta_0\sim\rho(\vTheta)$
			\STATE Sample $n$ trajectories in task $\mathcal{M}_{\vomega}$ under policy $\pi(\vtheta_0)$
			\STATE Estimate $j_{\vomega}(\vtheta_0)$, $\widehat{g}_N(\vtheta_0, \vomega)$
			\FOR{$t=0,\dots,T-1$}
			    \STATE Sample meta-action $h\in\mathcal{H}$
			    \STATE Update policy $\vtheta_{t+1}= \vtheta_t + h\frac{\widehat{g}_N(\vtheta_t, \vomega)}{\|\widehat{g}_N(\vtheta_t, \vomega)\|}$
			    \STATE Sample $n$ trajectories in $(\mathcal{M}_{\vomega},\pi(\vtheta_t))$
			    \STATE Estimate  $j_{\vomega}(\vtheta_{t+1})$, $\widehat{g}_N(\vtheta_{t+1}, \vomega)$
			    \STATE Set $x=\langle \vtheta_t, \widehat{g}_N(\vtheta_t,\vomega), \vomega\rangle$; \ \ $x'=\langle \vtheta_{t+1}, \widehat{g}_N(\vtheta_{t+1},\vomega), \vomega\rangle$; \ \  $ l=j_{\vomega}(\vtheta_{t+1})-j_{\vomega}(\vtheta_{t}).$
			    \STATE Append $\{(x,h,x',l)\}$ to $\mathcal{F}$
			\ENDFOR
		\ENDFOR
		\STATE \textbf{Output: } $\mathcal{F}$
	\end{algorithmic}
\end{algorithm}

\section{Fitted Q-Iteration on Meta-MDP}\label{sec:metaFQI}
We now define our approach to learn a dynamic learning rate in the framework of a meta-MDP. As a meta-RL approach, the objectives of our algorithm are to improve the generalization capabilities of PG methods and to remove the need to manually tune the learning rate for each task. Finding an optimal dynamic step size serves two purposes: it maximizes the convergence speed by performing large updates when allowed and improves the overall training stability by selecting low values when the return is close to the optimum or the current region is uncertain. 
To accomplish these goals, we propose the adoption of the Fitted Q-Iteration (FQI,\cite{ernst2005tree}) algorithm, which is an off-policy, and offline algorithm designed to learn a good approximation of the optimal action-value function by exploiting the Bellman optimality operator.
The approach consists in applying Supervised Learning techniques as, in our case, Extra Trees~\cite{geurts2006extremely}, in order to generalize the $Q$ estimation over the entire state-action space. The algorithm considers a full dataset $\mathcal{F} = \{(x_t^k,h_t^k,l_{t}^k,x_{t+1}^k)\}_k$, where each tuple represents an interaction with the meta-MDP: in the $k-$th tuple, $x^k_t$ and $x^k_{t+1}$ are respectively the current and next meta-state, $h^k_t$ the meta-action and $l^k_t$ the meta reward function, as described in Section \ref{sec:metaMDP}. To consider each meta-state $x$, there is the need to sample $n$ trajectories in the inner MDP to estimate return and gradient.
At the iteration $N$ of the algorithm, given the (meta) action-value function $Q_{N-1}$, the training set $TS_N =\{(i^k, o^k)\}_{k}$ is built, where each input is equivalent to the state-action pair $i^k = (x_t^k, h_t^k)$, and the target is the result of the Bellman optimal operator: $o^k = l_{t}^k + \widetilde{\gamma} \max_{h\in\mathcal{H}}Q_{N-1}(x_{t+1}^k, h)$. In this way, the regression algorithm adopted is trained on $TS$ to learn $Q_N$ with the learning horizon increased by one step. 

In general, the dataset is created by following $K$ learning trajectories over the CMDP: at the beginning of each meta-episode, a new context $\vomega$ and initial policy $\vtheta_0$ are sampled from $\psi$ and $\rho$; then, for each of the $T$ learning steps, the meta action $h$ is randomly sampled to perform the policy update. In this way, the overall dataset is composed of $KT$ tuples.
It is also possible to explore the overall task-policy space $\Omega\times\boldsymbol{\Theta}$ through a generative approach: instead of following the learning trajectories, both $\vomega, \vtheta_0$ and $h$ are sampled every time. 
We refer to this method as ``generative'' approach, while the former will be referred to as ``trajectory'' approach.
The pseudo-code for the dataset generation process with trajectories is provided in Algorithm \ref{alg:FQIdata}.

\vspace{0.15cm}
\textit{Double Clipped Q Function}\label{ssec:double}
As mentioned, each FQI iteration approximates the action-value function using the estimates made in the previous step. 
As the process goes on, the sequence of these compounding approximations can degrade the overall performance of the algorithm. In particular, FQI tends to suffer from overestimation bias, similarly to other value-based approaches that rely on taking the maximum of a noisy $Q$ function.
To countermeasure this tendency, we adopt a modified version of Clipped Double Q-learning, introduced by \cite{fujimoto2019offpolicy}, to penalize uncertainties over future states. This approach consists in maintaining two parallel functions $Q_N^{\{1,2\}}$ for each iteration and choosing the action $h$ maximizing a convex combination of the minimum and the maximum between them:
$$
l+\tilde{\gamma} \max _{h\in\mathcal{H}}\left[\lambda \min _{j=1,2} Q^j\left(x^{\prime}, h\right)+(1-\lambda) \max _{j=1,2} Q^j\left(x^{\prime}, h\right)\right],
$$
with $\lambda>0.5$. If we set $\lambda=1$, the update corresponds to Clipped Double Q-learning. The minimum operator penalizes high variance estimates in regions of uncertainty and pushes the policy towards actions that lead to states already seen in the dataset.

The overall procedure introduces external hyperparameters,e.g. the number of decision trees, the minimum number of samples for a split (\textit{min split}), and $\lambda$. However, the sensitivity on these parameters is minimal \cite{geurts2006extremely}, as a different set of hyperparameters does not impact the ability of FQI to converge.

\begin{figure*}[t]
	\includegraphics[width=.99\textwidth]{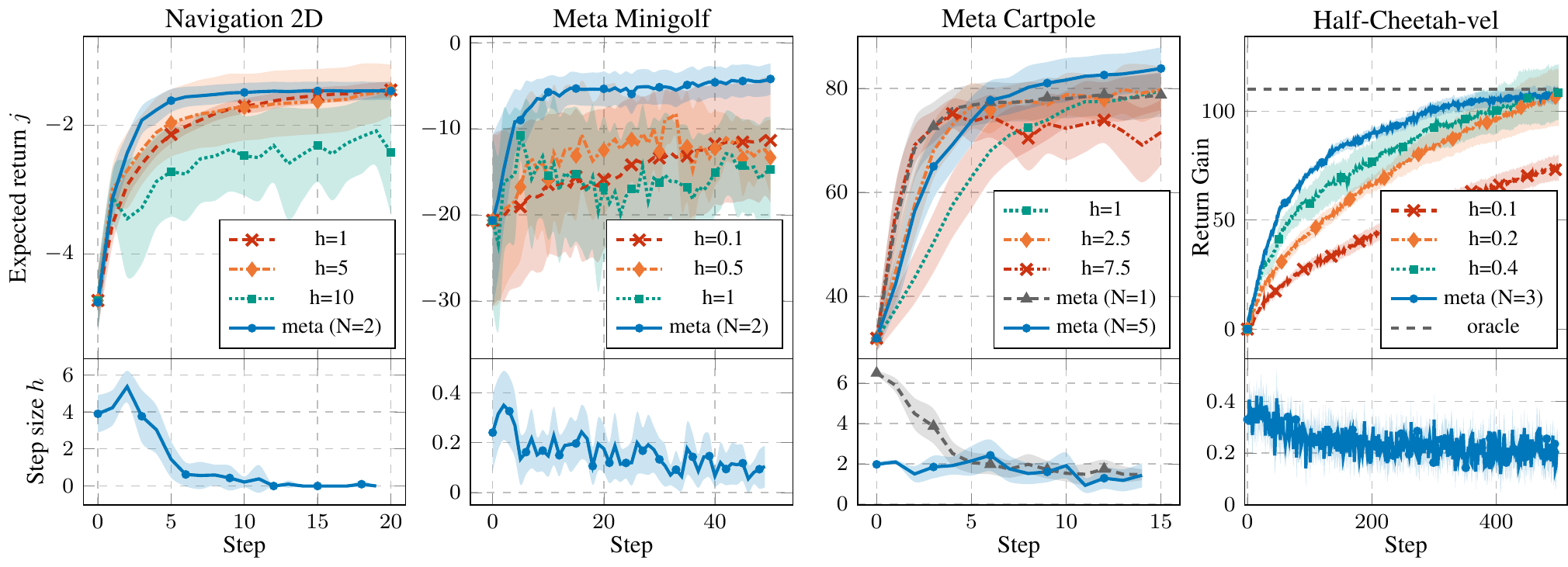}
	\caption{FQI model performance against NGA with fixed step size $h$. The top plots show the expected returns or the return gain. The bottom plots show the meta actions chosen through learning iterations. $N$ represents the FQI iteration selected. (20 runs/random test contexts, avg $\pm$ 95 \% c.i.)}
	\label{fig:learning}
\end{figure*}


\section{Experimental Evaluation} \label{sec:experiments}

In this section, we show an empirical analysis of the performance of our approach in different environments. As we shall see, the meta action can choose the best step size and dynamically adapt it to fine-tune the learning procedure.
As FQI iterations proceed, new estimation errors are gradually introduced, resulting in overfitting the model (with the \textit{target loss} minimized on the training dataset), and consequently in degradation of out-of-sample performances over time. This is due to the error propagation w.r.t. the optimal $Q-$value function in the whole state-action space (and task space, in our case), as in \cite{farahmand2010error}. As a consequence, the model iterations are evaluated in a validation process, as in the standard model selection procedure, on a set of out-of-sample tasks and policies. From this set, the model obtaining the best mean return, said $N$ is selected.
The results of the selected models are shown in Figure \ref{fig:learning}, along with NGA performed with fixed step size, tested on the same 20 trials (i.e., on the same random test tasks and initial policies), and performed with the same batch size for each trial. Our code is based upon OpenAI Gym \cite{gym} and Baselines \cite{baselines} toolkits.

\vspace{0.15cm}
\textit{Navigation2d:} For our first evaluation of the approach, we consider one of the environments presented in \cite{finn2017modelagnostic}, called Navigation2D. This environment consists of a unit square space in which an agent aims to reach a random goal in the plane.
The distribution of the tasks implemented is such that, at each episode, a different goal point is uniformly selected in the unit square.
As we can note in the left plots of Figure \ref{fig:learning}, the algorithm can select large step sizes with a good starting return gain without suffering from any drop. The algorithm can calibrate its action, starting with larger improvements and slowing down once the policy gets good results. In addition, all trajectories reach convergence in fewer steps than any other method.

\vspace{0.15cm}
\textit{Minigolf:}
In our second experiment, inspired by \cite{penner2002physics,tirinzoni19a}, we consider the scenario of a flat minigolf green, in which the agent has to hit the ball with a putter and place the ball inside the hole in the minimum number of strokes. The CMDP is built by varying the putter length and the friction coefficient. The environment is Lipschitz w.r.t. the context, but it is the only framework where the reward is non-Lipschitz, since for each step it can be either 0 if the shot is a success, -1 if the ball does not reach the goal (and the episode continues) or -100 for overshooting.
The central plot in Figure \ref{fig:learning} illustrates the performance of our approach in the same set of random test tasks. 
We can see that the algorithm can consistently reach the optimal values by choosing an adaptive step size. In addition, the convergence to the global optimum is achieved in around 10 meta steps of training, a substantial improvement w.r.t. the choice of a fixed learning rate, which leads (when it converges) to a local minimum, meaning constantly undershooting until the end of the episode. 

\vspace{0.15cm}
\textit{CartPole:}
For our third experiment, we examine the CartPole balancing task \cite{barto83cartpole}, which consists of a pole attached to a cart, where the agent has to move to balance the pole as long as possible.
The CMDP is induced by varying the pole mass and length.
To be more focused on the very first steps, and to better generalize on the overall policy and task space, the training dataset was built considering trajectories with only 15 total updates.
To have a fair comparison, the right plots of Figure \ref{fig:learning} illustrate an evaluation of the approach in the selected environment, where we have tested the resulting  FQI model (and NGA with fixed step sizes) performing the same number of total updates as the training trajectories.\footnote{Being this environment an alteration of the classic Cartpole, standard results cannot be compared.} In the supplementary materials, we provide further results, where the models are tested for a longer horizon $T=60$ and show the results closer to convergence.
Differently from before, it is possible to see that the best model (blue solid line) is choosing to update the policy with small learning rates: this leads to a lower immediate return gain (high rates have a better learning curve in the first steps) but allows to improve the overall meta return. This is because the model is planning with a horizon of $N=5$ policy updates. Indeed, we included also the results of the first FQI iteration, which tries to optimize the immediate gain. As expected, the agent selects high step sizes for the first iterations, obtaining high immediate rewards only in the first learning steps.


\vspace{0.15cm}
\textit{Half-cheetah with goal velocity}: 
As a last environment, we considered the half-cheetah locomotion problem introduced in \cite{finn2017modelagnostic} with MuJoCo simulator \cite{todorov2012mujoco}, where a planar cheetah has to learn to run with a specific goal velocity. This is the most complex environment among the ones presented as the policy, albeit linear, is composed of 108 parameters.
From the rightmost plot of Figure \ref{fig:learning} we can see the performance gain $j(\vtheta_t)-j(\vtheta_0)$.\footnote{The expected return changes deeply w.r.t. the task $\vomega$, hence the learning curves as in the other plots in Figure \ref{fig:learning} show very high variance, independently from the robustness of the models.} The FQI model, trained with NGA trajectories with $T=500$ total updates, is learning faster than benchmarks. The interesting fact is that the meta actions chosen by the model are within the range [0.2, 0.4], while the curves obtained with a fixed learning rate within those values are not able to obtain the same return gains. In the figure, we provide also the \textit{oracle} value, as provided in \cite{finn2017modelagnostic}. 

\begin{figure*}[t]
\centering
	\includegraphics[width=1\textwidth]{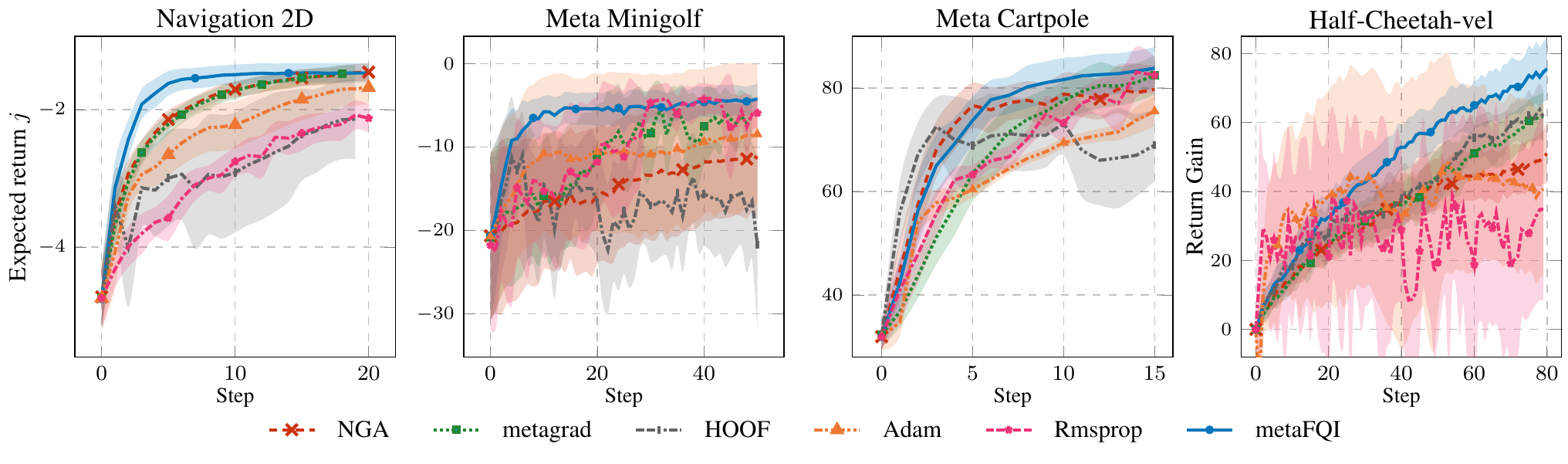}
	\caption{FQI performance comparison against benchmarks (20 runs, 95\% c.i.).}
	\label{fig:learning_dec}
\end{figure*}
\vspace{0.15cm}
\textit{Benchmark comparison. } In Figure \ref{fig:learning} we compared our approach with the choice of a fixed step size. There are, of course, many different schedules and optimization algorithms for the choice of the learning rate, and among the most widely adopted there are RMSprop and Adam \cite{kingma2014adam}. The former considers an adaptive learning rate by introducing a momentum term and normalizing the step direction through a moving average of the square gradient. Adam, instead, takes also advantage of the exponentially decaying average of the second moments of the gradients.
We compared our results (\textit{metaFQI}) against tuned implementations of the mentioned update rules, and against the best fixed stepsize (\textit{NGA}). Moreover, we include in the comparison also two other benchmarks for learning rate adaptation: \textit{HOOF, \cite{Supratik2019Fast}} and \textit{metagrad, \cite{xu2018meta}}, which have been implemented to optimize the stepsize for NGA (more details in the supplementary material).
the results are shown in Figure \ref{fig:learning_dec}, in the same settings as the ones provided in Section \ref{sec:experiments}. The only difference, for reasons of computational times, is the horizon of the Half-Cheetah environment, reduced to $T=80$. We see that our approach outperforms the previous methods, showing improved learning with, in general, lower variance in the returns obtained. Moreover, all the considered benchmarks heavily rely on the initial stepsize chosen and on the outer meta-hyperparameters, which deeply affect the learning capabilities.

\section{Discussion, Limitations and Future Work}\label{sec:conc}
In this paper, we considered the problem of hyperparameter tuning for policy gradient-based algorithms in Contextual Markov Decision Processes, where heterogeneous contexts may require different solutions.
In particular, we modeled the general problem through the meta-MDP definition, for which any policy-based update rule can be optimized using learning as reward. We analyzed the case of Lipschitz meta-MDPs, deriving some general guarantees that hold if the model is smooth with respect to the parameterization of the context and the transition processes.
Finally, we implemented the Fitted Q-Iteration algorithm on the meta-MDP where the update rule is the Natural Gradient Ascent, and we used it to choose an adaptive step size through the learning process.
The approach has been evaluated in different settings, where we observed good generalization capabilities of the model, which can reach fast convergence speed and robustness without the need for manual hyperparameter tuning.

Many challenges can be addressed in future work for this approach to be effective in real-life applications. First of all, more complex environments can be considered, and we can extend this method to different update rules and hyperparameters. One direct extension of our approach can be applied to the choice of the max Kullback-Leibler divergence constraints in Trust-Region-based approaches \cite{schulman2015trust,schulman2017proximal}: some results in this direction can already be observed in \cite{occorso22a}. Moreover, the main limitation of our current approach is the same as for many hyperparameter tuning approaches: the computational time required to build the training dataset. One possible way to improve the sample efficiency might consist in evaluating the meta-reward by means of importance sampling, as in \cite{Supratik2019Fast}. In realistic settings, where deep policies are required, the inclusion of all policy parameters in the meta-state might be inefficient; a solution might consist in compressing the representation of the policy through autoencoders, or through the choice of specific informative meta-features: in this way, our approach would be independent on the policy architecture and scalable for large domains.

\subsection*{Ethical Statement}
Hyperparameter selection for policy-based algorithms has a significant impact on the ability to learn a highly performing policy in Reinforcement Learning, especially with heterogeneous tasks, where different contexts may require different solutions. Our approach shows that it is possible to learn an automatic selection of the best configurations that can be identified after a manual fine-tuning of the parameters. Consequently, our work can be seen as a further step in the AutoML direction, in which a practitioner could run the algorithm and, with some guidance, obtain optimal performance in just a few steps without the need for manual fine-tuning. Beyond this, we are not aware of any societal consequences of our work, such as welfare, fairness, or privacy.

\bibliographystyle{splncs04}
\bibliography{biblio}

\appendix
\onecolumn
\section{Proofs}\label{app:proofs}
In this part of the appendix, we provide the proofs of the results shown in the main paper.
\subsection{Lipschitz continuity of the action-value function}
Before describing the proof for Theorem \ref{thm:lip_q_omega}, we need to recall the Bellman Operator $T^\pi$ for the Action Value Function $Q_{\vomega}^\pi$:
\begin{align*}
       T^\pi Q_{\vomega}^{\pi}(s,a)&=R_{\vomega}(s,a) + \gamma \int_\Sspace P_{\vomega}(s'|s,a)\int_\Aspace Q_{\vomega}^{\pi}(s',a')\pi(a'|s')dads'\\
       &=R_{\vomega}(s,a) + \gamma \int_\Sspace P_{\vomega}(s'|s,a)V_{\vomega}^{\pi}(s')ds'
\end{align*}
where $Q_{\vomega}^\pi$ is the fixed point. 

\noindent
Moreover, let's consider a preliminary result on the LC-continuity of the value functions presented in \cite{pirotta2015policy} :
\begin{restatable}[Lipschitz value functions]{lemma}{}\label{lem:Q_lip}
Given an $(L_P, L_R)$-LC MDP and a $L_\pi$-LC stationary policy $\pi$, if $\gamma L_P(1+L_\pi)<1$, then the Q-function $Q^\pi$ is $L_{Q^\pi}$-LC and the V function is $L_{V^\pi}$-LC \wrt the joint state-action space;
\begin{equation*}
    L_{Q^\pi} = \frac{L_R}{1-\gamma L_P(1+L_\pi)}; \qquad L_{V^\pi} = L_{Q^\pi}(1+L_\pi)
\end{equation*}
\end{restatable} 

\lipschitzQ*
\begin{proof}
We follow the same ideas as in \cite{oatao17977}: first of all, given an $L_{\omega_{Q}}$-LC continuous Q function $Q^{\pi}$ w.r.t. the task space $\Omega$, the related value function $V_{\vomega}^{\pi}$ is $L_{\omega_{Q}}$-LC. Indeed,
\begin{equation*}
    \begin{aligned}
    \biggl|V_{\vomega}^{\pi}(s)-V_{\widehat{\vomega}}^{\pi}(s)\biggr| &=\left|\int_{\mathcal{A}} \pi(a \mid s)\left(Q_{\vomega}^{\pi}(s, a)-Q_{\widehat{\vomega}}^{\pi}(s, a)\right) d a\right| \\
    & \leq \int_{\mathcal{A}} \pi(a \mid s)\biggl|Q_{\vomega}^{\pi}(s, a)-Q_{\widehat{\vomega}}^{\pi}(s, a)\biggr| d a \\
    & \leq \max _{a}\biggl|Q_{\vomega}^{\pi}(s, a)-Q_{\widehat{\vomega}}^{\pi}(s, a)\biggr| \leq L_{\omega_{Q}} d_{\Omega}(\vomega, \widehat{\vomega}).
    \end{aligned}
\end{equation*}
Now, we consider the iterative application of Bellman Operators, in such a way that $Q_{\vomega}^{\pi,n+1}=T^\pi Q_{\vomega}^{\pi,n}$, and we prove that $Q_{\vomega}^{\pi,n}$ is $L_{\omega_{Q}}^{n}$-LC continuous, and that satisfies the recurrence relation:
\begin{equation}\label{eq:rec_L}
    L_{\omega_{Q}}^{n+1}=L_{\omega_{r}}+\gamma L_{\pi} L_{V}(\vomega)+\gamma L_{\omega_{Q}}^{n}.
\end{equation}
Indeed, for $n=1$ the property holds immediately, since:
\begin{equation*}
    \left|Q_{\vomega}^{\pi, 1}(s, a)-Q_{\widehat{\vomega}}^{\pi, 1}(s, a)\right|=\Bigl|R_{\vomega}(s, a)-R_{\widehat{\vomega}}(s, a)\Bigr| \leq L_{\omega_{r}} d_{\Omega}(\vomega, \widehat{\vomega}).
\end{equation*}
Now, let us suppose the property holds for $n$. Then:
\begin{align*}
    \biggl|Q_{\vomega}^{\pi, n+1}(s, a)-Q_{\widehat{\vomega}}^{\pi, n+1}(s, a)\biggr| =\nonumber\\ 
    \biggl\lvert R_{\vomega}(s, a)-R_{\widehat{\vomega}}(s, a) 
  +\gamma \int_{\mathcal{S}} P_{\vomega}\left(s^{\prime} \mid s, a\right) V_{\vomega}^{\pi, n}\left(s^{\prime}\right) d s^{\prime} 
  -\gamma \left.\int_{\mathcal{S}} P_{\widehat{\vomega}}\left(s^{\prime} \mid s, a\right) V_{\widehat{\vomega}}^{\pi}\left(s^{\prime}\right) d s^{\prime} \right\rvert \nonumber\\
 \leq L_{\omega_{r}} d_{\Omega}(\vomega, \widehat{\vomega}) 
 +\gamma\left|\int_{\mathcal{S}}\left(P_{\vomega}\left(s^{\prime} \mid s, a\right)-P_{\widehat{\vomega}}\left(s^{\prime} \mid s, a\right)\right) V_{\vomega}^{\pi, n}\left(s^{\prime}\right) d s^{\prime}\right|\nonumber\\
  +\gamma\left|\int_{\mathcal{S}} P_{\widehat{\vomega}}\left(s^{\prime} \mid s, a\right)\left(V_{\vomega}^{\pi, n}\left(s^{\prime}\right)-V_{\widehat{\vomega}}^{\pi, n}\left(s^{\prime}\right)\right) d s^{\prime}\right| \nonumber\\
 \leq L_{\omega_{r}} d_{\Omega}(\vomega, \widehat{\vomega})+\gamma L_{V}(\vomega) \sup _{\|f\|_{L} \leq 1}\left\{\left|\int_{\mathcal{S}}\left(P_{\vomega}\left(s^{\prime} \mid s, a\right)\right.\right.\right. 
  - \left. P_{\widehat{\vomega}}\left(s^{\prime} \mid s, a\right)\right) f\left(s^{\prime}\right) d s^{\prime}\biggr|\biggr\} \nonumber\\
  +\gamma \max _{s^{\prime}}\biggl|V_{\vomega}^{\pi, n}\left(s^{\prime}\right)-V_{\widehat{\vomega}}^{\pi, n}\left(s^{\prime}\right)\biggr| \nonumber\\
 \leq \left(L_{\omega_{r}}+\gamma L_{\omega_{P}} L_{V}(\vomega)+\gamma L_{\omega_{Q}}^{n}\right) d_{\Omega}(\vomega, \widehat{\vomega}).\nonumber
\end{align*}
Consequently, Inequality \ref{eq:rec_L} holds. Now, if the sequence $L_{\omega_{Q}}^{n}$ is convergent, it converges to the fixed point of the recurrence equation:
\begin{equation*}
    L_{\omega_{Q}}=L_{\omega_{r}}+\gamma L_{\omega_{P}} L_{V}(\vomega)+\gamma L_{\omega_{Q}}.
\end{equation*}
Hence the limit point is the one expressed in Equation \ref{eq:L_Qw}, and the sequence can be proven to be convergent since $\gamma<1$.
\end{proof}
As a consequence, the Proof that $j_{\vomega}(\pi)$ is CLC under $\vomega$ is immediate:
\begin{equation*}
    \begin{aligned}
    \biggl|j_{\vomega}(\pi)-j_{\widehat{\vomega}}(\pi)\biggr| &=\left|\int_{\mathcal{S}} \mu\left(s_{0}\right)\left[V_{\vomega}^{\pi}\left(s_{0}\right)-V_{\widehat{\vomega}}^{\pi}\left(s_{0}\right)\right] ds_{0}\right| \\
    & \leq \int_{\mathcal{S} \times \mathcal{A}} \mu\left(s_{0}\right) \pi\left(a \mid s_{0}\right)\Bigl|Q_{\vomega}^{\pi}\left(s_{0}, a\right)-Q_{\widehat{\vomega}}^{\pi}\left(s_{0}, a\right)\Bigr| da ds_{0} \\
    & \leq L_{\omega_{Q}}(\pi) d_{\Omega}(\vomega, \widehat{\vomega}).
    \end{aligned}
\end{equation*}

\subsection{Lipschitz Continuity of the gradient}
In order to consider the Lipschitz continuity of the gradient of the return $\nabla j_{\vomega}(\vtheta)$, we first need to introduce two more assumptions:
\begin{restatable}[]{ass}{Lipschitz Parametric Policy}\label{as:theta_LC}
Let $\pi_{\boldsymbol{\theta}} \in \Pi$ be a policy parametrized in the parameters space $\boldsymbol{\theta}$. A LC-policy $\pi$ satisfies the following conditions:
\begin{align}
    \forall \boldsymbol{\theta} \in \boldsymbol{\Theta}, \forall s, s^{\prime} \in \mathcal{S} &\quad \mathcal{K}\left(\pi_{\boldsymbol{\theta}}(\cdot \mid s), \pi_{\boldsymbol{\theta}}\left(\cdot \mid s^{\prime}\right)\right) \leq L_{\pi_{\boldsymbol{\theta}}} d_{\mathcal{S}}\left(s, s^{\prime}\right) \label{eq:s-piLC}\\
    \forall s \in \mathcal{S}, \forall \boldsymbol{\theta}, \boldsymbol{\theta}^{\prime} \in \boldsymbol{\Theta} &\quad \mathcal{K}\left(\pi_{\boldsymbol{\theta}}(\cdot \mid s), \pi_{\boldsymbol{\theta}^{\prime}}(\cdot \mid s)\right) \leq L_{\pi}(\boldsymbol{\theta}) d_{\boldsymbol{\Theta}}\left(\boldsymbol{\theta}, \boldsymbol{\theta}^{\prime}\right).\label{eq:theta-piLC}
\end{align}
\end{restatable}

\begin{restatable}[]{ass}{Lipschitz Gradient of Policy Logarithm}\label{as:grad_LC}
The gradient of the policy logarithm must satisfy the conditions of:
\begin{enumerate}
    \item Uniformly bounded gradient: $\forall(s, a) \in \mathcal{S} \times \mathcal{A}, \forall \boldsymbol{\theta} \in \boldsymbol{\Theta}, \forall i=1, \ldots, d$
    \begin{equation*}
        \left|\nabla_{\boldsymbol{\theta}_{i}} \log \pi_{\boldsymbol{\theta}}(a \mid s)\right| \leq M_{\boldsymbol{\theta}}^{i};
    \end{equation*}
    \item  State-action LC: $\forall\left(s, s^{\prime}, a, a^{\prime}\right) \in \mathcal{S}^{2} \times \mathcal{A}^{2}, \forall \boldsymbol{\theta} \in \boldsymbol{\Theta}, \forall i=1 \ldots, d$
    \begin{equation*}
        \left.\left|\nabla_{\boldsymbol{\theta}_{i}} \log \pi_{\boldsymbol{\theta}}(a \mid s)-\nabla_{\boldsymbol{\theta}_{i}} \log \pi_{\boldsymbol{\theta}}\left(a^{\prime} \mid s^{\prime}\right)\right| \leq L^i_{\nabla \log \pi} d_{\mathcal{S} \times \mathcal{A}}\left((s, a),\left(s^{\prime}, a^{\prime}\right)\right)\right);
    \end{equation*}
    \item Parametric LC: $\forall\left(\boldsymbol{\theta}, \boldsymbol{\theta}^{\prime}\right) \in \boldsymbol{\Theta}, \forall(s, a) \in \mathcal{S} \times \mathcal{A}, \forall i=1, \ldots, d$
    \begin{equation*}
        \left|\nabla_{\boldsymbol{\theta}_{i}} \log \pi_{\boldsymbol{\theta}}(a \mid s)-\nabla_{\boldsymbol{\theta}_{i}} \log \pi_{\boldsymbol{\theta}^{\prime}}(a \mid s)\right| \leq L_{\nabla \log \pi}^{i}(\boldsymbol{\theta}) d_{\boldsymbol{\Theta}}\left(\boldsymbol{\theta}, \boldsymbol{\theta}^{\prime}\right).
    \end{equation*}
\end{enumerate}
\end{restatable}

\begin{lemma} [Lemma 3 from \cite{pirotta2015policy}]
Given Assumptions \ref{ass:LipMDP}, \ref{as:theta_LC}, if $\gamma L_{P}(1+$ $\left.L_{\pi_{\boldsymbol{\theta}}}\right)<1$, the Kantorovich distance between a pair of $\gamma$-discounted feature state distributions is Parametric-LC (PLC) w.r.t. parameters $\boldsymbol{\theta}: \forall(\boldsymbol{\theta}, \widehat{\boldsymbol{\theta}}) \in \boldsymbol{\Theta}^{2}$:
\begin{equation}
    \mathcal{K}\left(\delta_{\mu}^{\boldsymbol{\theta}}, \delta_{\mu}^{\widehat{\boldsymbol{\theta}}}\right) \leq L_{\delta}(\boldsymbol{\theta}) d_{\boldsymbol{\Theta}}(\boldsymbol{\theta}, \widehat{\boldsymbol{\theta}});
\end{equation}
where $ L_{\delta}(\boldsymbol{\theta})=\frac{\gamma L_{P} L_{\pi}(\boldsymbol{\theta})}{1-\gamma L_{P}\left(1+L_{\pi_{\boldsymbol{\theta}}}\right)}$.
\end{lemma}
\noindent In the same fashion, we can now define the state occupancy measure in the task $\mathcal{M}_{\vomega}$ as $\delta_{\mu,\vomega}^{\vtheta}$ and
we can prove the following lemma:\footnote{Assumption \ref{as:theta_LC} is not entirely required, but only Inequality \ref{eq:s-piLC} is required to hold.}
\begin{lemma}[L-continuity of meta state occupancy measures] Given Assumptions \ref{ass:LipMDP}, \ref{as:theta_LC} and \ref{as:lip_metamdp}, if $\gamma L_{P}(\vomega)\left(1+L_{\pi_{\boldsymbol{\theta}}}\right)<1$, then the Kantorovich distance between a pair of $\gamma$-discounted feature-state distributions is CLC w.r.t. context $\vomega$:
\begin{equation}
    \mathcal{K}\left(\delta_{\mu, \vomega}^{\boldsymbol{\theta}}, \delta_{\mu, \widehat{\vomega}}^{\boldsymbol{\theta}}\right) \leq L_{\delta}(\vomega) d_{\Omega}(\vomega, \widehat{\vomega}), \quad \forall(\vomega, \widehat{\vomega}) \in \Omega^{2};
\end{equation}
where $L_{\delta}(\vomega)=\frac{\gamma L_{\omega_{P}}}{1-\gamma L_{P}(\vomega)\left(1+L_{\pi_{\boldsymbol{\theta}}}\right)}$.
\end{lemma}
\begin{proof}
\begin{align}
    &\mathcal{K}\left(\delta_{\mu, \vomega}^{\boldsymbol{\theta}}, \delta_{\mu, \widehat{\vomega}}^{\boldsymbol{\theta}}\right)=\sup _{f}\left\{\left|\int_{\mathcal{S}}\left(\delta_{\mu, \vomega}^{\boldsymbol{\theta}}(s)-\delta_{\mu, \widehat{\vomega}}^{\boldsymbol{\theta}}(s)\right) f(s) d s\right|:\|f\|_{L} \leq 1\right\} \nonumber\\
    &=\sup _{f}\left\{\left|_{\mathcal{S}}\left(\mu(s)+\gamma \int_{\mathcal{S}} \int_{\mathcal{A}} \pi_{\boldsymbol{\theta}}\left(a \mid s^{\prime}\right) P_{\vomega}\left(s \mid s^{\prime}, a\right) \delta_{\mu, \vomega}\left(s^{\prime}\right) d a d s^{\prime}\right) f(s)-\right.\right.\nonumber\\
    & \quad \left.-\left(\mu(s)+\gamma \int_{\mathcal{A}} \int_{\mathcal{S}} \pi_{\boldsymbol{\theta}}\left(a \mid s^{\prime}\right) P_{\widehat{\vomega}}\left(s \mid s^{\prime}, a\right) \delta_{\mu, \widehat{\vomega}}\left(s^{\prime}\right) d a d s^{\prime}\right) f(s) d s \biggr|:\|f\|_{L} \leq 1\right\}\nonumber\\
    &= \gamma \sup _{f:\|f\|_{L} \leq 1}\left\{\left|\int_{\mathcal{S}} f(s) \int_{\mathcal{A}} \int_{\mathcal{S}}\left(P_{\vomega}\left(s \mid s^{\prime}, a\right) \pi_{\boldsymbol{\theta}}\left(a \mid s^{\prime}\right) \delta_{\mu, \vomega}^{\boldsymbol{\theta}}\left(s^{\prime}\right)\right.\right.\right.\nonumber\\
    & \quad .\left. -P_{\widehat{\vomega}}\left(s \mid s^{\prime}, a\right) \pi_{\boldsymbol{\theta}}\left(a \mid s^{\prime}\right) \delta_{\mu, \widehat{\vomega}}^{\boldsymbol{\theta}}\left(s^{\prime}\right)\right) ds^{\prime} da ds\biggr|\biggr\} \nonumber\\
    &= \gamma \sup _{f:\|f\|_{L} \leq 1}\left\{\biggl| \int_{\mathcal{S}} f(s) \int_{\mathcal{A}} \int_{\mathcal{S}} P_{\vomega}\left(s \mid s^{\prime}, a\right) \pi_{\boldsymbol{\theta}}\left(a \mid s^{\prime}\right)\left(\delta_{\mu, \vomega}^{\boldsymbol{\theta}}\left(s^{\prime}\right)-\delta_{\mu, \widehat{\vomega}}^{\boldsymbol{\theta}}\left(s^{\prime}\right)\right) d s^{\prime} d a d s\right.\nonumber\\
    & \quad \left.+\int_{\mathcal{S}} f(s) \int_{\mathcal{A}} \int_{\mathcal{S}}\left(P_{\vomega}\left(s \mid s^{\prime}, a\right)-P_{\widehat{\vomega}}\left(s \mid s^{\prime}, a\right)\right) \pi_{\boldsymbol{\theta}}\left(a \mid s^{\prime}\right) \delta_{\mu, \widehat{\vomega}}^{\boldsymbol{\theta}}\left(s^{\prime}\right) d s^{\prime} d a d s \biggr|\right\}\nonumber\\
    &\leq \underbrace{\gamma \sup _{f:\|f\|_{L} \leq 1}\left\{\left|\int_{\mathcal{S}}\left(\delta_{\mu, \vomega}^{\boldsymbol{\theta}}\left(s^{\prime}\right)-\delta_{\mu, \widehat{\vomega}}^{\boldsymbol{\theta}}\left(s^{\prime}\right)\right) \int_{\mathcal{A}} \pi_{\boldsymbol{\theta}}\left(a \mid s^{\prime}\right) \int_{\mathcal{S}} P_{\vomega}\left(s \mid s^{\prime}, a\right) f(s) d s d a d s^{\prime}\right|\right\}}_{(1)} \nonumber\\
    & \quad +\gamma \underbrace{\sup _{f:\|f\|_{L} \leq 1}\left\{\left|\int_{\mathcal{S}} \delta_{\mu, \widehat{\vomega}}^{\boldsymbol{\theta}}\left(s^{\prime}\right) \int_{\pi}^{\boldsymbol{\theta}}\left(a \mid s^{\prime}\right) \int_{\mathcal{S}}\left(P_{\vomega}\left(s \mid s^{\prime}, a\right)-P_{\widehat{\vomega}}\left(s \mid s^{\prime}, a\right)\right) f(s) d s d a d s^{\prime}\right|\right\}}_{(2)}.
\end{align}
Now, we focus on term (1):
\begin{align}
&\sup _{f:\|f\|_{L} \leq 1}\biggl\{\biggl|\int_{\mathcal{S}}\left(\delta_{\mu, \vomega}^{\boldsymbol{\theta}}\left(s^{\prime}\right)-\delta_{\mu, \widehat{\vomega}}^{\boldsymbol{\theta}}\left(s^{\prime}\right)\right) \underbrace{\int_{\mathcal{A}} \pi_{\boldsymbol{\theta}}\left(a \mid s^{\prime}\right) \int_{\mathcal{S}} P_{\vomega}\left(s \mid s^{\prime}, a\right) f(s) d s d a}_{h_{f, \vomega}^{\boldsymbol{\theta}}\left(s^{\prime}\right)} d s^{\prime}\biggr| \biggr\} \nonumber \\
&=L_{P}(\vomega)\left(1+L_{\pi_{\boldsymbol{\theta}}}\right) \sup _{f:\|f\|_{L} \leq 1}\left\{\left|\int_{\mathcal{S}}\left(\delta_{\mu, \vomega}^{\boldsymbol{\theta}}\left(s^{\prime}\right)-\delta_{\mu, \widehat{\vomega}}^{\boldsymbol{\theta}}\left(s^{\prime}\right)\right) \frac{h_{f, \vomega}^{\boldsymbol{\theta}}\left(s^{\prime}\right)}{L_{P}(\vomega)\left(1+L_{\pi_{\boldsymbol{\theta}}}\right)} d s^{\prime}\right|\right\} \nonumber \\
&\leq L_{P}(\vomega)\left(1+L_{\pi_{\boldsymbol{\theta}}}\right) \sup _{\tilde{f}:\|\widetilde{f}\|_{L} \leq 1}\left\{\left|\int_{\mathcal{S}}\left(\delta_{\mu, \vomega}^{\boldsymbol{\theta}}\left(s^{\prime}\right)-\delta_{\mu, \widehat{\vomega}}^{\boldsymbol{\theta}}\left(s^{\prime}\right)\right) \tilde{f}\left(s^{\prime}\right) d s^{\prime}\right|\right\} \label{eq:delta_LC}\\
&\leq L_{P}(\vomega)\left(1+L_{\pi_{\boldsymbol{\theta}}}\right) \mathcal{K}\left(\delta_{\mu, \vomega}^{\boldsymbol{\theta}}, \delta_{\mu, \widehat{\vomega}}^{\boldsymbol{\theta}}\right).\nonumber
\end{align}
where Inequality \ref{eq:delta_LC} comes from the fact that $h_{f, \vomega}^{\boldsymbol{\theta}}\left(s^{\prime}\right):=\int_{\mathcal{A}} \pi_{\boldsymbol{\theta}}\left(a \mid s^{\prime}\right) \int_{\mathcal{S}} P_{\vomega}\left(s \mid s^{\prime}, a\right) f(s) d s d a$
is $L_{P}(\vomega)\left(1+L_{\pi_{\boldsymbol{\theta}}}\right)$-PLC w.r.t. the state space $\mathcal{S}$.
From the other side, the term (2) can be bounded as follows:
\begin{equation*}
    \begin{aligned}
    & \sup _{f:\|f\|_{L} \leq 1}\left\{\left|\int_{\mathcal{S}} \delta_{\mu, \widehat{\vomega}}^{\boldsymbol{\theta}}\left(s^{\prime}\right) \int_{\mathcal{A}} \pi_{\boldsymbol{\theta}}\left(a \mid s^{\prime}\right) \int_{\mathcal{S}}\left(P_{\vomega}\left(s \mid s^{\prime}, a\right)-P_{\widehat{\vomega}}\left(s \mid s^{\prime}, a\right)\right) f(s) d s d a d s^{\prime}\right|\right\} \\
    & \leq \int_{\mathcal{S}} \delta_{\mu, \widehat{\vomega}}^{\boldsymbol{\theta}}\left(s^{\prime}\right) \int_{\mathcal{A}} \pi_{\boldsymbol{\theta}}\left(a \mid s^{\prime}\right) \sup _{f:\|f\|_{L} \leq 1}\left\{\left|\int_{\mathcal{S}}\left(P_{\vomega}\left(s \mid s^{\prime}, a\right)-P_{\widehat{\vomega}}\left(s \mid s^{\prime}, a\right)\right) f(s) d s\right|\right\} d a d s^{\prime} \\
    & \leq \int_{\mathcal{S}} \delta_{\mu, \widehat{\vomega}}^{\boldsymbol{\theta}}\left(s^{\prime}\right) \int_{\mathcal{A}} \pi_{\boldsymbol{\theta}}\left(a \mid s^{\prime}\right) \mathcal{K}\left(P_{\vomega}\left(\cdot \mid s^{\prime}, a\right), P_{\widehat{\vomega}}\left(\cdot \mid s^{\prime}, a\right)\right) d a d s^{\prime} \\
    & \leq L_{\omega_{P}} d_{\Omega}(\vomega, \widehat{\vomega}).
    \end{aligned}
\end{equation*}
Finally, merging everything together:
\begin{equation*}
    \begin{aligned}
    \mathcal{K}\left(\delta_{\mu, \vomega}^{\boldsymbol{\theta}}, \delta_{\mu, \widehat{\vomega}}^{\boldsymbol{\theta}}\right) & \leq \gamma L_{P}(\vomega)\left(1+L_{\pi_{\boldsymbol{\theta}}}\right) \mathcal{K}\left(\delta_{\mu, \vomega}^{\boldsymbol{\theta}}, \delta_{\mu, \widehat{\vomega}}^{\boldsymbol{\theta}}\right)+\gamma L_{\omega_{P}} d_{\Omega}(\vomega, \widehat{\vomega}) \\
    & \leq \frac{\gamma L_{\omega_{P}}}{1-\gamma L_{P}(\vomega)\left(1+L_{\pi_{\boldsymbol{\theta}}}\right)} d_{\Omega}(\vomega, \widehat{\vomega}).
    \end{aligned}
\end{equation*}
\end{proof}
As a direct consequence, we define the joint probability  $\zeta(\delta_{\mu,\pi},\pi)$ between the state distribution $\delta_{\mu,\pi}$ and the stationary policy $\pi$. In the case of a parametric policy and a $\vomega$-based MDP, we will denote it as $\zeta_{\mu,\vomega}^{\vtheta}$.It is then easy to prove that:
\begin{equation*}
    \mathcal{K}\left(\zeta_{\mu, \vomega}^{\boldsymbol{\theta}}, \zeta_{\mu, \widehat{\vomega}}^{\boldsymbol{\theta}}\right) \leq L_{\delta}(\vomega)\left(1+L_{\pi_{\boldsymbol{\theta}}}\right) d_{\Omega}(\vomega, \widehat{\vomega}).
\end{equation*}
\begin{proof}
\begin{align}
\mathcal{K}\left(\zeta_{\mu, \vomega}^{\boldsymbol{\theta}}, \zeta_{\mu, \widehat{\vomega}}^{\boldsymbol{\theta}}\right) &= \sup _{f:\|f\|_{L} \leq 1}\left\{\left\lVert \int_{\mathcal{S}} \delta_{\mu, \vomega}^{\boldsymbol{\theta}}(s) \int_{\mathcal{A}} \pi_{\boldsymbol{\theta}}(a \mid s) f(s, a) da ds\right.\right. \nonumber\\
& \quad \left.\left. -\int_{\mathcal{S}} \delta_{\mu, \widehat{\vomega}}^{\boldsymbol{\theta}}(s) \int_{\mathcal{A}} \pi_{\boldsymbol{\theta}}(a \mid s) f(s, a) da ds \right\rVert\right\} \nonumber\\
&= \sup _{f:\|f\|_{L} \leq 1}\left\{\norm{\int_{\mathcal{S}}\left(\delta_{\mu, \vomega}^{\boldsymbol{\theta}}(s)-\delta_{\mu, \widehat{\vomega}}^{\boldsymbol{\theta}}(s)\right) \int_{\mathcal{A}} \pi_{\boldsymbol{\theta}}(a \mid s) f(s, a) da ds}\right\} \nonumber\\
&\leq \left(1+L_{\pi_{\boldsymbol{\theta}}}\right) \mathcal{K}\left(\delta_{\mu, \vomega}^{\boldsymbol{\theta}}, \delta_{\mu, \widehat{\vomega}}^{\boldsymbol{\theta}}\right)\label{eq:zeta_LC}.
\end{align}
where in \ref{eq:zeta_LC} we used the fact that, for a function $f$ defined on $\mathcal{S} \times \mathcal{A}$ such that $\|f\|_{L} \leq 1$, then $\int_{\mathcal{A}} \pi_{\boldsymbol{\theta}}(a \mid s) f(s, a) d a$ is $\left(1+L_{\pi_{\boldsymbol{\theta}}}\right)$-LC.
\end{proof}
\begin{lemma}[L-continuity of $\eta$]
Given Assumptions \ref{ass:LipMDP}, \ref{as:lip_metamdp}, \ref{as:theta_LC} and \ref{as:grad_LC},  $\eta_{i, \vomega}^{\boldsymbol{\theta}}(s, a)\coloneqq\nabla_{\boldsymbol{\theta}_{i}} \log \pi_{\boldsymbol{\theta}}(s, a) Q^{\boldsymbol{\theta}}_{\vomega}(s, a)$ is L-CLC w.r.t. the context $\vomega$:
\begin{equation*}
    \begin{aligned}
    \left|\eta_{i, \vomega}^{\boldsymbol{\theta}}(s, a)-\eta_{i, \widehat{\vomega}}^{\boldsymbol{\theta}}(s, a)\right| &=\left|\nabla_{\boldsymbol{\theta}_{i}} \log \pi_{\boldsymbol{\theta}}(a \mid s)\left(Q_{\vomega}^{\boldsymbol{\theta}}(s, a)-Q^{\boldsymbol{\theta}}_{\widehat{\vomega}}(s, a)\right)\right| \\
    & \leq \mathcal{M}_{\boldsymbol{\theta}}^{i}\left|Q_{\vomega}^{\boldsymbol{\theta}}(s, a)-Q_{\widehat{\vomega}}^{\boldsymbol{\theta}}(s, a)\right| \\
    & \leq \mathcal{M}_{\boldsymbol{\theta}}^{i} L_{\omega_{Q}} d_{\Omega}(\vomega, \widehat{\vomega}).
    \end{aligned}
\end{equation*}
Moreover, $\eta$ is also L-LC w.r.t. the joint state-action space $\mathcal{S} \times \mathcal{A}$:
\begin{equation*}
    \left|\eta_{i, \vomega}^{\boldsymbol{\theta}}(s, a)-\eta_{i, \vomega}^{\boldsymbol{\theta}}(\widehat{s}, \widehat{a})\right| \leq L_{\eta^{\boldsymbol{\theta}}(\vomega)}^{i} d_{\mathcal{S} \times \mathcal{A}}((s, a),(\widehat{s}, \widehat{a}));
\end{equation*}
where $L_{\eta^{\boldsymbol{\theta}}(\vomega)}^{i}= \frac{R_{\max }}{1-\gamma} L_{\nabla \log \pi_{\boldsymbol{\theta}}}^{i}+\mathcal{M}_{\boldsymbol{\theta}}^{i} L_{Q^{\boldsymbol{\theta}}(\vomega)}$.
\end{lemma}

\begin{restatable}[]{thr}{grad_J_lip}\label{thm:lc_gradient_j}
Finally, given Assumptions \ref{ass:LipMDP}, \ref{as:lip_metamdp}, \ref{as:theta_LC} and \ref{as:grad_LC}, the return gradient is CLC w.r.t. the context $\vomega$:
\begin{equation*}
    \left|\nabla_{\boldsymbol{\theta}_{i}} j_{ \vomega}(\boldsymbol{\theta})-\nabla_{\boldsymbol{\theta}_{i}} j_{ \widehat{\vomega}}(\boldsymbol{\theta})\right| \leq L_{\nabla j}(\vomega) d_{\Omega}(\vomega, \widehat{\vomega});
\end{equation*}
where $L_{\nabla j}(\vomega)=L_{\eta_{\omega}^{\boldsymbol{\theta}}}^{i}\left(1+L_{\pi_{\boldsymbol{\theta}}}\right) L_{\delta}(\vomega)+\mathcal{M}_{\boldsymbol{\theta}}^{i} L_{\omega_{Q}}$.
\end{restatable}
\begin{proof}
\begin{equation*}
    \begin{aligned}
    &\left|\nabla_{\boldsymbol{\theta}_{i}} j_{ \vomega}(\boldsymbol{\theta})-\nabla_{\boldsymbol{\theta}_{i}} j_{ \widehat{\vomega}}(\boldsymbol{\theta})\right|=\biggl|\underset{(s, a) \sim \zeta_{\mu, \vomega}^{\boldsymbol{\theta}}}{\mathbb{E}}\left[\eta_{i, \vomega}^{\boldsymbol{\theta}}(s, a)\right]-\underset{(s, a) \sim \zeta_{\mu, \widehat{\vomega}}^{\boldsymbol{\theta}}}{\mathbb{E}}\left[\eta_{i, \widehat{\vomega}}^{\boldsymbol{\theta}}(s, a)\right] \biggr| \\
    &\leq\left|\int_{\mathcal{S}} \int_{\mathcal{A}}\left(\zeta_{\mu, \vomega}^{\boldsymbol{\theta}}-\zeta_{\mu, \widehat{\vomega}}^{\boldsymbol{\theta}}\right)(s, a) \eta_{i, \vomega}^{\boldsymbol{\theta}}(s, a) d a d s\right|+\biggl|\underset{(s, a) \sim \zeta_{\mu, \widehat{\vomega}}^{\boldsymbol{\theta}}}{\mathbb{E}}\left[\eta_{i, \vomega}^{\boldsymbol{\theta}}(s, a)-\eta_{i, \widehat{\vomega}}^{\boldsymbol{\theta}}(s, a)\right] \biggr| \\
    &\leq L_{\eta^{\boldsymbol{\theta}}(\vomega)}^{i} \mathcal{K}\left(\zeta_{\mu, \vomega}^{\boldsymbol{\theta}}, \zeta_{\mu, \widehat{\vomega}}^{\boldsymbol{\theta}}\right)+\left|\int_{\mathcal{S}} \delta_{\mu, \widehat{\vomega}}^{\boldsymbol{\theta}}(s) \int_{\mathcal{A}} \pi_{\boldsymbol{\theta}}(a \mid s)\left(\eta_{i, \vomega}^{\boldsymbol{\theta}}(s, a)-\eta_{i, \widehat{\vomega}}^{\boldsymbol{\theta}}(s, a)\right) d a d s\right|\\
    &\leq\left[L_{\eta_{\omega}^{\boldsymbol{\theta}}}^{i}\left(1+L_{\pi_{\boldsymbol{\theta}}}\right) L_{\delta}(\vomega)+\mathcal{M}_{\boldsymbol{\theta}}^{i} L_{\omega_{Q}}\right] d_{\Omega}(\vomega, \widehat{\vomega}).
    \end{aligned}
\end{equation*}
\end{proof}

\section{Experiment Details}\label{app:exps}
In this section, we provide more details regarding the experimental campaign provided. In the following environments, all the policies considered are Gaussian, and linear w.r.t. the state observed (with bias $\theta_0$), i.e. $\pi_{\vtheta}(a|s)\sim\mathcal{N}(\theta_0 + \vtheta^\top s, \sigma^2)$, where $\sigma$ is fixed standard deviation, with a different setting for each environment.

\vspace{0.15cm}
\textit{Infrastructure}
The experiments have been run on a machine with two CPUs Intel(R) Xeon(R) CPU E7-8880 v4 @2.20GHz (22 cores, 44 thread, 55 MB cache) and 128 GB RAM.

\subsection{Navigation2D Description}\label{subapp:nav2d}
The Navigation2D environment consists of a 2-dimensional square space in which an agent, represented as a point, aims to reach a goal in the plane traversing the minimum distance. 

At the start of the episode, the agent is placed in the initial position $s_{0} = (0, 0)$. Then, at each step $t$ the agent observes its current position and performs an action $a_{t}$ corresponding to movement speeds along the $x$ and $y$ axes:
\begin{equation}\label{eq6.12}
    a_{t} = (v_{x}, v_{y}), \text{ where } v_{x}, v_{y} \in [-v^{\max}, v^{\max}].
\end{equation}
According to this action, the agent can move in every direction of the plane, with a limit on the maximum speed $v^{\max}=0.1$ allowed in a single step. This parameters determines the minimum number of steps necessary to reach the goal and can be varied to tune the difficulty of the environment.

At each step, the environment produces a reward equal to the negative Euclidean distance from the goal:
\begin{equation}\label{eq6.13}
    r_{t} = \sqrt{(x_t - x_{\text{goal}})^2 + (y_t - y_{\text{goal}})^2}.
\end{equation}
An episode terminates when the agent is within a threshold distance $d_{\text{thresh}}$ from the goal or when the horizon $H=10$ is reached.

The distribution of tasks is implemented as a CMDP $\mathcal{M}(\vomega)$ in which, at each episode, a different goal point is selected at random. The context $\vomega$ is given by a 2D vector, such that:
\begin{equation}
    \vomega=(x_{\text{goal}}, y_{\text{goal}}), \text{ where } x_{\text{goal}}, y_{\text{goal}} \sim U(-1, 1).
\end{equation}

Parameters used for experiments:
\begin{itemize}
\itemsep0em
    \item initial policy distribution $\rho = \mathcal{N}(0, 0.1)$;
    \item discount factor $\gamma=0.99$;
    \item policy standard deviation $\sigma=1.001$; 
    \item task distribution $\psi = \mathcal{U}([-0.5,0.5]^2)$;
    \item meta-discount factor $\widetilde{\gamma}=1$;
    \item FQI dataset method: trajectories;
    \item FQI number of samples: $K=4000$ with learning horizon $T=20$;
    \item inner trajectories $n=200$ with horizon $H=10$; 
    \item number of estimators = 50, minimum samples split = 0.01;
    \item step size $\mathcal{H}= [0, 8]$;
    \item step size sampling distribution: uniform in $\mathcal{H}$;
    \item step size selected in evaluation from an evenly spaced discretization of 101 values in $\mathcal{H}$.
\end{itemize}

\subsection{Minigolf Description}\label{subapp:minigolf}
In the minigolf game, the agent has to shoot a ball with radius $r$ inside a hole of diameter $D$ with the smallest number of strokes. 
The friction imposed by the green surface is modeled by a constant deceleration $d=\frac{5}{7} \rho g$, where $\rho$ is the dynamic friction coefficient between the ball and the ground and $g$ is the gravitational acceleration.
Given the distance $x$ of the ball from the hole, the agent must choose the force $a$, from which the velocity of the ball $v$ of the ball is determined as $v=al^2(1+\epsilon)$, where $\epsilon\sim\mathcal{N}(0,0.25)$ and $l$ is the putter length.
 For each distance $x$, the ball falls in the hole if its velocity $v$ ranges from $v_{min}=\sqrt{2dx}$ to $v_{max}=\sqrt{(2D-r)^2\frac{g}{2r}+v^2_{min}}$. In this case, the episode ends with a reward 0;
 if $v > v_{max}$ the ball falls outside the green, and the episode ends with a reward -100. Otherwise, if $v<v_{min}$, the agents gets a reward equal to -1, and the episode goes on from a new position $x_{new} = x_{old}-\frac{v^2}{2d}$. 
At the beginning of each episode, the initial position is selected from an uniform distribution between 0m and 20m from the hole. 
The stochasticity of the action implies that the stronger is the action chosen the more uncertain is the outcome, as the effect of r.v. $\epsilon$ become more effective. As a result, when it is away from the hole, the agent might not prefer to try to make a hole in one shot, preferring to perform a sequence of closer shots. In this case, the context is given by the friction coefficient $\rho\in[0.065, 0.196]$ and by the putter length $l\in[0.7,1]m$.
 
During the experiment, the environment parameters are set to imitate the dynamics of a realistic shot in a minigolf green, within the limits of our simplified simulation. This is the complete configuration adopted:
\begin{itemize}
\itemsep0em
    \item horizon $H = 20$;
    \item discount factor $\gamma=0.99$;
    \item angular velocity $\vomega \in  [1 \times 10^{-5}, 10]$;
    \item initial distance $x_{0} \in [0, 20] \text{ meters}$;
    \item ball radius $r = 0.02135 \text{ meters}$;
    \item hole diameter $D = 0.10 \text{ meters}$;
    \item gravitational acceleration $g=9.81 \frac{\text{meters}}{\text{second}^2}$.
\end{itemize}
The distribution of tasks is built as a CMDP $\mathcal{M}(\vomega)$, induced by the pair $\vomega = (l, \rho)$. At each meta episode, a new task is sampled from a multivariate uniform distribution within this ranges:
\begin{itemize}
\itemsep0em
    \item putter length $l \sim U(0.7, 1) \text{ meters}$;
    \item friction coefficient $\rho \sim U(0.065,0.196)$.
\end{itemize}

Parameters used for experiments:
\begin{itemize}
\itemsep0em
    \item initial policy distribution $\boldsymbol{\theta}=(w,b) \sim U((-1,2),(-2,3.5))$ (2-dimensional policy);
    \item policy standard deviation $\sigma=0.1$; 
    \item meta-discount factor $\widetilde{\gamma}=1$;
    \item FQI dataset method: generative;
    \item FQI number of samples: $K=10000$;
    \item inner trajectories $n=400$ with horizon $H=20$; 
    \item number of estimators = 50, minimum samples split = 0.01;
    \item step size space: $\mathcal{H} = [0, 1]$
    \item step size sampling distribution: uniform in $\mathcal{H}$;
    \item step size selected in evaluation from an evenly spaced discretization of 101 values in $\mathcal{H}$.
\end{itemize}

\subsection{CartPole description}\label{subapp:cartpole}
The CartPole environment \cite{barto83cartpole}, also known as the Inverted Pendulum problem, consists in a pole attached to a cart by a non actuated joint, making it an inherently unstable system. The cart can move horizontally along a frictionless track to balance the pole. The objective is to maintain the equilibrium as long as possible.

In this implementation, an episode starts with the pendulum in vertical position. At each step, the agent observes the following 4-tuple of continuous values:
\begin{itemize}
\itemsep0em
    \item cart position $x_{\text{cart}} \in [-4.8, 4.8]$;
    \item cart velocity $v_{\text{cart}} \in \mathbb{R}$;
    \item pole angle $\phi_{\text{pole}} \in [-0.418, 0.418]$ rad;
    \item pole angular velocity $\vomega_{\text{pole}} \in \mathbb{R}$.
\end{itemize}
Given the state, the agent chooses an action between 0 and 1 to push the cart to the left or to the right. For each step in which the pole is in balance, the environment produces a reward of +1. An episode ends when the pole angle from the vertical position is higher than 12 degrees, or the cart moves more than 2.4 units from the center, or the horizon $H=100$ is reached.

In our experiments, we set the environment parameters to these values:
\begin{itemize}
\itemsep0em
\item mass of the cart $m_{cart} = 1$ kg;
\item length of the pole $l_{pole} = 0.5$ m;
\item force applied by the cart $F = 10$ N.
\end{itemize}
The CMDP $\mathcal{M}(\vomega)$ is induced by varying two environment parameters, the pole mass $m_{pole}$ and the pole length $l_{pole}$, that form the context parameterization $\vomega=(m_{pole}, l_{pole})$. Each task in the meta-MDP is built by sampling $\vomega$ from a multivariate uniform distribution, within these ranges:
\begin{itemize}
\itemsep0em
    \item pole length $l_{pole} \sim U(0.5,1.5)\mathrm{m}$;
    \item pole mass $m_{pole} \sim U(0.1, 2)$ kg.
\end{itemize}

Parameters used for experiments:
\begin{itemize}
\itemsep0em 
    \item initial policy distribution $\boldsymbol{\theta_d} \sim \mathcal{N}(0, 0.01)$ for each component $\vtheta_d$;
    \item policy standard deviation $\sigma=1.001$; 
    \item meta-discount factor $\widetilde{\gamma}=1$;
    \item FQI dataset method: trajectories;
    \item FQI number of samples: $K=3200$ with learning horizon $T=15$;
    \item inner trajectories $n=100$ with horizon $H=100$; 
    \item number of estimators = 150, minimum samples split = 0.05;
    \item step size $\mathcal{H}=[0, 10]$;
    \item step size sampling distribution: uniform in $\mathcal{H}$;
    \item step size selected in evaluation from an evenly spaced discretization of 101 values in $\mathcal{H}$.
\end{itemize}

\subsection{Half Cheetah description}\label{subapp:half_cheetah}
The CMDP $\mathcal{M}(\vomega)$ is induced by varying the goal velocity of the half cheetah $v_{goal}$, which defines the context $\vomega$, with uniform distribution $U(0,2)$.

Parameters used for experiments:
\begin{itemize}
\itemsep0em
    \item initial policy distribution $\boldsymbol{\theta_d} \sim \mathcal{N}(0, 0.1)$ for each component $\vtheta_d$;
    \item policy standard deviation $\sigma=1.001$; 
    \item meta-discount factor $\widetilde{\gamma}=1$;
    \item FQI dataset method: trajectories;
    \item FQI number of samples: $K=200$ with learning horizon $T=500$ ($T=80$ for the comparison against benchmarks);
    \item inner trajectories $n=100$ with horizon $H=100$; 
    \item number of estimators = 150, minimum samples split = 0.05;
    \item step size $\mathcal{H}=[0, 1]$;
    \item step size sampling distribution: uniform in $\mathcal{H}$;
    \item step size selected in evaluation from an evenly spaced discretization of 101 values in $\mathcal{H}$.
\end{itemize}

\subsection{Metagrad implementation}
In Figure \ref{fig:learning_dec}, we provided the comparison of the proposed approach with meta-gradient \cite{xu2018meta}. The algorithm performs an online update following the gradient of the (differentiable) return function w.r.t. the hyperparameter set.
In particular, after updating the current parameter set $\vtheta$ to $\vtheta'$ following the update rule $f(\vtheta, h, \tau)$, the hyperparameter gradient of the return function on a new batch of trajectories $\tau'$ is approximated as in Equation \ref{eq:metagrad}:
\begin{equation}\label{eq:metagrad}
\frac{\partial J(\vtheta', h,\tau')}{\partial h} = \frac{\partial J(\vtheta', h,\tau')}{\partial \vtheta'}\frac{d \vtheta'}{dh}\approx \frac{\partial J(\vtheta', h,\tau')}{\partial \vtheta'} z',
\end{equation}
where $z\approx \frac{d\vtheta}{dh}$ is update as an accumulative trace with parameter $\mu$:
$$ z' = \mu z + \frac{\partial f(\vtheta, h, \tau)}{\partial h}.$$
In the original paper, the optimized hyperparameters (as well as in \cite{yu2006fast} with HOOF implementation) were the discount factor $\gamma$ and the exponential weight coefficient $\lambda$ related to the generalized advantage estimation. In our case, the hyperparameter considered is the stepsize. 

Hence, from the NGA update function it follows that:
\begin{align*}
    \vtheta' &= \vtheta + h \frac{\widehat{g}_N(\vtheta, \vomega)}{\|\widehat{g}_N(\vtheta, \vomega)\|_2}\\
    z' &= \mu z + \frac{\widehat{g}_N(\vtheta, \vomega)}{\|\widehat{g}_N(\vtheta, \vomega)\|_2}.
\end{align*}

After the update, the stepsize is update through a meta-hyperparameter $\beta$ and a new batch of trajectories with policy $\pi_{\vtheta'}$:

\begin{align*} 
h' &= h-\beta \frac{\widehat{g}_N(\vtheta', \vomega)}{\|\widehat{g}_N(\vtheta', \vomega)\|_2}z' \\
&= h- \beta \frac{\widehat{g}_N(\vtheta', \vomega)}{\|\widehat{g}_N(\vtheta', \vomega)\|_2}(\mu z + \frac{\widehat{g}_N(\vtheta, \vomega)}{\|\widehat{g}_N(\vtheta, \vomega)\|_2})
\end{align*}
 In the case $\mu=0$, as adopted in \cite{xu2018meta}, the stepsize update function reduces to computing the cosine similarity between consecutive gradients:
 $$h'= h-\beta\ \  \text{sim} \Big(\widehat{g}_N(\vtheta', \vomega),\widehat{g}_N(\vtheta, \vomega)\Big),$$
where $\text{sim}(x,y)$ denotes the cosine similarity between vectors $x$ and $y$.


\section{Other results}
In this section of the appendix, we provide more experimental results.

\subsection*{Meta Cartpole SwingUp}
For the experimental session, a variant of the Cartpole presented in Section \ref{sec:experiments} is the Cartpole Swingup variant, introduced in \cite{tornio2006variational} and implemented in \cite{duan2016benchmarking}. The main difference is the following: classic CartPole environment provides unitary reward per step until the end of the episode, which ends when the pole angle from the vertical axis $\phi_{pole}$ is more than 12 degrees from vertical, or the cart moves more than $x_{thresh}=2.4$ units from the center. CartPole Swingup, instead, has a reward equal to $\cos(\phi_{pole})$,  and equal to $-100$ if the cart threshold $x_{thresh}=3$ is reached.
Finally, the CMDP in this case is built by changing only the pole mass $m_{pole}$ with a uniform distribution $\sim U(0.1,2)$.

Parameters used for experiments:
\begin{itemize}
\itemsep0em 
    \item initial policy distribution $\boldsymbol{\theta_d} \sim \mathcal{N}(0, 0.1)$ for each component $\vtheta_d$;
    \item policy standard deviation $\sigma=1.001$; 
    \item meta-discount factor $\widetilde{\gamma}=1$;
    \item FQI dataset method: trajectories;
    \item FQI number of samples: $K=300$ with learning horizon $T=25$;
    \item inner trajectories $n=100$ with horizon $H=200$; 
    \item number of estimators = 150, minimum samples split = 0.05;
    \item step size $\mathcal{H}=[0, 0.5]$;
    \item step size sampling distribution: uniform in $\mathcal{H}$;
    \item step size selected in evaluation from an evenly spaced discretization of 101 values in $\mathcal{H}$.
\end{itemize}

The results are depicted in Figure \ref{fig:swingup}: the model chosen is the one which maximizes the reward, i.e. $N=1$: however, all iterations have similar performances, which resemble the choice of a fixed learning rate equal to 0.1. Indeed, the actions chosen are almost always around this value, with the exception of the first two steps, where higher step sizes are taken into account (with a slightly better     learning).

\begin{figure}[t]
\centering
	\includegraphics[width=.99\textwidth]{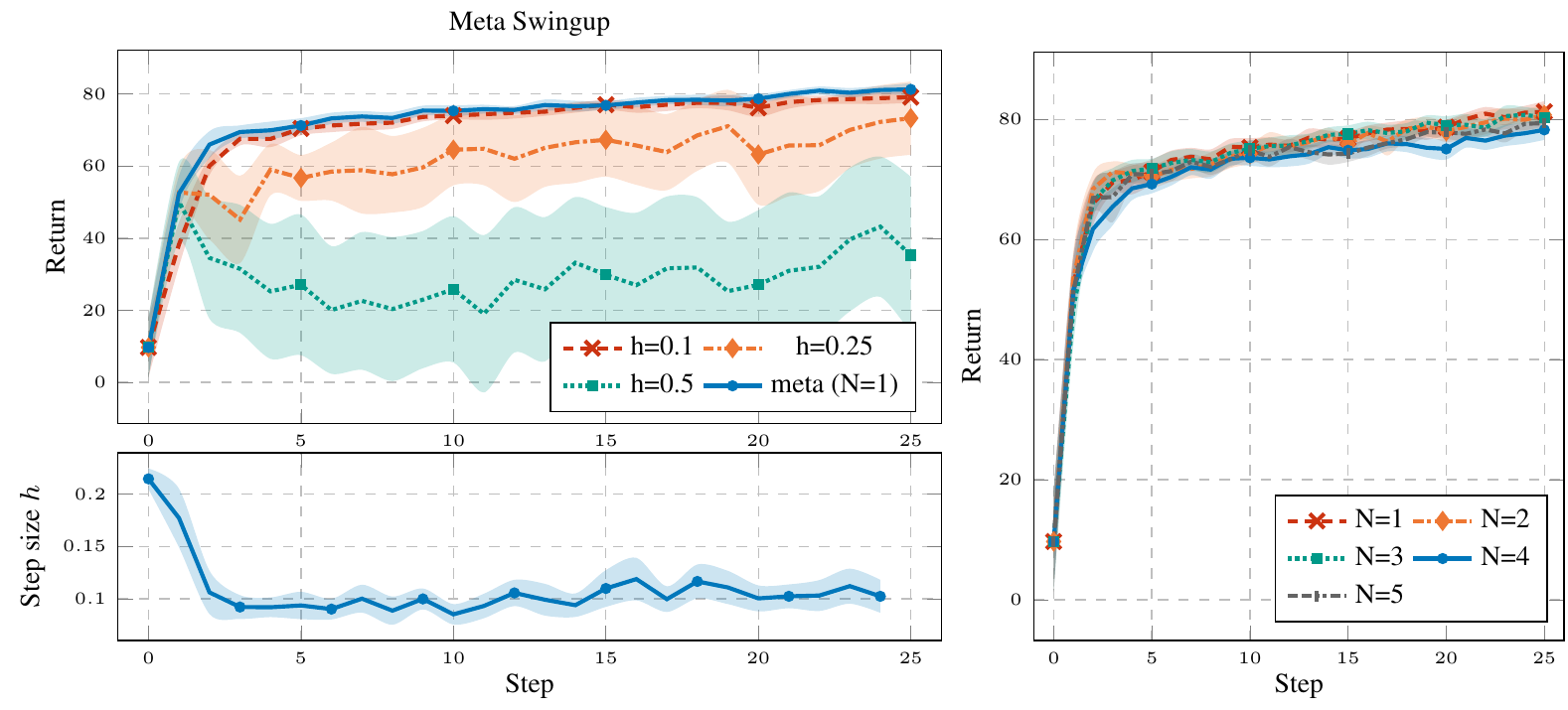}
	\caption{Meta Swingup FQI model performance on 20 random test context and initial policies against fixed step sizes. The top left plot shows the 95\% confidence intervals of the expected returns. The bottom left plot shows the  meta action chosen through learning iterations. $N$ represents the FQI iteration selected. The right plot shows the performance among different FQI iterations.}
	\label{fig:swingup}
\end{figure}

\subsection*{Comparison among FQI Iterations.}
As said, as the regression procedures are iterated in the application of FQI algorithm, there is a trade-off between a larger planning horizon and the accumulation of new regression errors.
In Figure \ref{fig:fqi_it} we show some of the learning curves with different FQI iterations. For all the environments considered, it is possible to see that the direct regression on the meta reward (i.e. one FQI iteration) does not provide the best performances, while from a certain point the results start to get worse. As far as Meta Cartpole environment is concerned, we can clearly see that the models select progressively more cautious steps in order to improve learning, as explained in Section \ref{sec:experiments}.

\begin{figure}[t]
\centering
	\includegraphics[width=.99\textwidth]{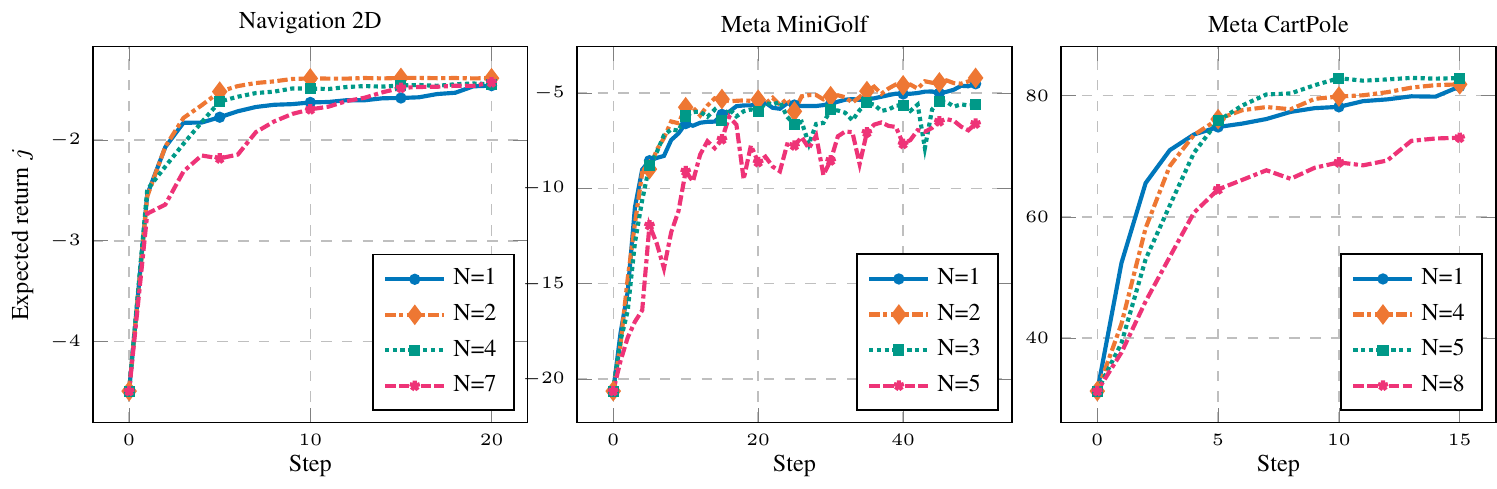}
	\caption{FQI model performance among different iterations. For the sake of clarity, only the average values are shown.}
	\label{fig:fqi_it}
\end{figure}

\subsection*{Comparison with learning rate schedules: details } 
In Figure \ref{fig:learning_dec}, we compared our approach with three different baselines, where the initial learning rate (denoted as $\alpha$) was tuned by grid search on 20 random test contexts and initial policies, and the best were selected for the comparison. 
Adam and RMSprop updates have a poor performance when applied to the natural gradient $g(\boldsymbol{\theta})$, hence they have been tuned by adopting the (standard) stochastic gradient $\widehat{\nabla}_N j(\boldsymbol{\theta})$. 
While for the decaying learning rate the only hyperparameter is the initial rate, the other methods depend also on other variables, which were kept fixed to the suggested values: for RMSProp, the parameters were fixed as $\rho=0.9, \epsilon=1e-7$, while for Adam the parameters were fixed as $\beta_1=0.9, \beta_2=0.999, \epsilon=1e-7$. The notation for these parameters follows the one used in the implementations of the optimizers within Python Keras API \cite{gulli2017deep}, used to perform the updates. 

As far as meta-gradient is concerned, the algorithm shows heavy dependence on the initial stepsize $h_0$ selected, while the impact of the meta-stepsize $\beta$ is reduced. $\mu$ is always set to 0.  
HOOF has been tested by selecting the KL-constraints$\epsilon$ in the set $[0.0001,0.001,0.01,0.02,0.05]$, and by sampling $Z=100$ candidate hyperparameters per iteration in the meta-action space. 

Another common stepsize schedule adopted to grant convergence is a  decaying step size $h_{t+1}=\frac{\alpha}{t}$ (similar as is an exponentially decreasing learning rate $h_{t+1} = \alpha h_t$), where $h_0$ is the initial learning rate. The comparison of the model trained through FQI and this baselines is shown in Figure \ref{fig:decay}, while the best initial learning rates chosen for each of the environments and of the baselines (and shown in Figure \ref{fig:learning_dec}) can be found in Table \ref{tab:bench}.
\begin{figure}[t]
\centering
	\includegraphics[width=\textwidth]{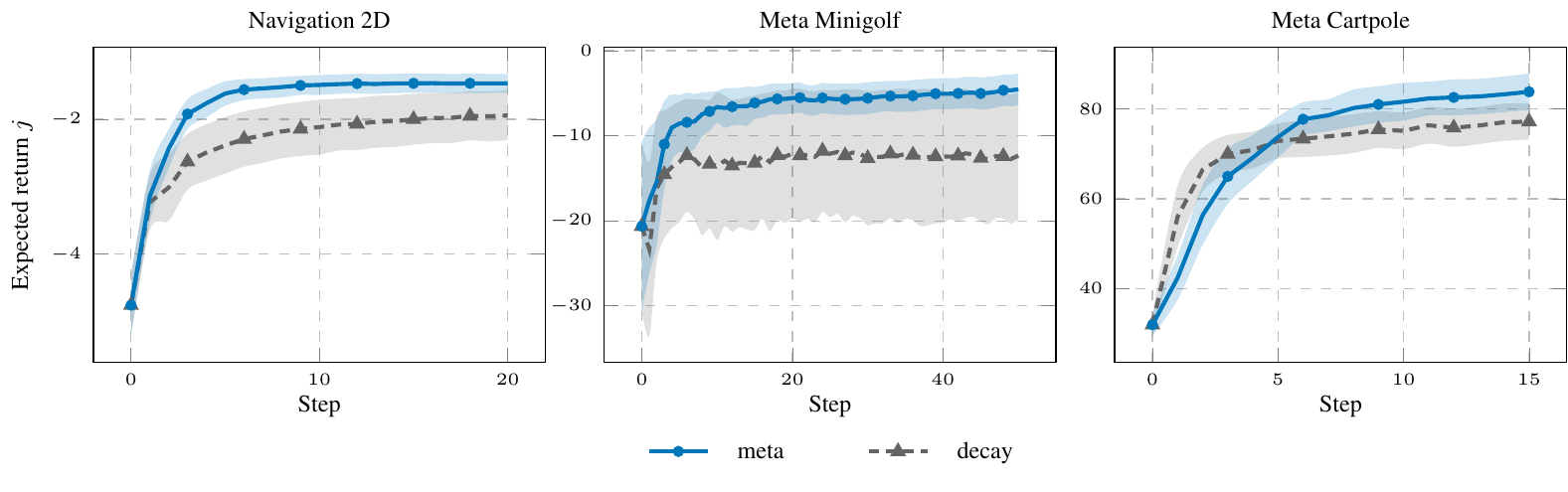}
	\caption{FQI model performance against an exponentially decreasing (\textit{decaying}) learning rate.  20 runs, avg$\pm$ 95\% c.i.}
	\label{fig:decay}
\end{figure}

\begin{table*}[t]
\caption{Best initial learning rate selected. Evaluation using 20 different random tasks and policies.}
\label{tab:bench}
\begin{center}
\scriptsize
\begin{tabular}{ccccc}
\toprule
 & Nav2D & Meta MiniGolf & MetaCartpole  & Half-Cheetah\\
\midrule
 \textit{RMSprop} & 0.9 & 0.3 & 0.3 & 0.3\\
 \textit{Adam} & 0.8 & 0.08 & 0.3 & 0.5\\
 \textit{Metagrad: $h_0$} & 3 & 0.3 & 1 & 0.5\\
 \textit{Metagrad: $\beta$} & 0.001 & 5 & 0.1 & 0.01\\
 \textit{decay} & 5 & 2 & 7.5 & N/A\\
\bottomrule
\end{tabular}
\end{center}
\vspace{-.4cm}
\end{table*}

\subsection*{Cartpole: Extension of trajectory length}
In Figure \ref{fig:learning} we have shown the performance of FQI models trained on Cartpole trajectories with horizon $T=15$ update steps. In order to have a fair comparison, we have tested the resulting FQI model (and NGA with fixed step sizes) performing the same number of total updates as the training trajectories. However, as the learning curves were far from convergence, one may ask what happens if the horizon is increased: Figure \ref{fig:cart_extended} depicts the performance of the same models (trained on $T=15$-steps long trajectories) with an increased horizon of $60$ steps.
\begin{figure}[t]
\centering
	\includegraphics[width=.5\textwidth]{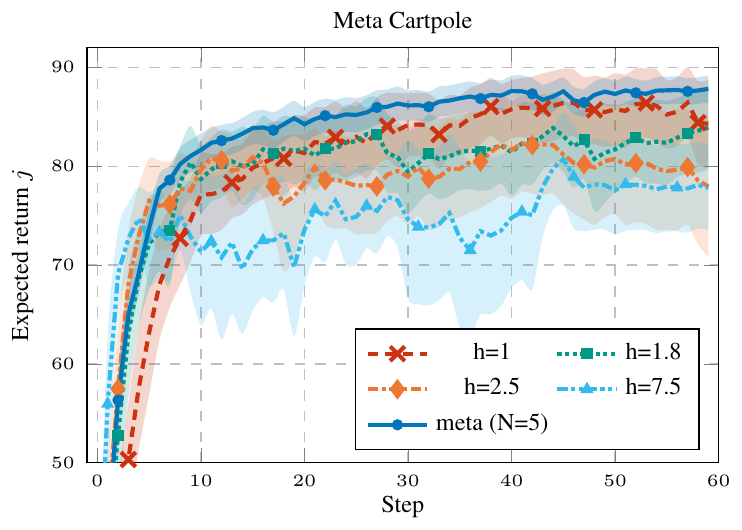}
	\caption{FQI model performance in Cartpole against NGA with fixed size $h$. Train with $T=15$ steps. Test on 60 updates. The red dashed line represents the best NGA model, with $h=1.8$. 20 runs, avg$\pm$ 95\% c.i.}
	\label{fig:cart_extended}
\end{figure}

\subsection*{Explicit knowledge of the context: is it informative?}
In the experimental campaign, we assumed to be able to represent the parametrized context $\vomega$, as this information can be used to achieve an \textit{implicit task-identification} by the agent. However, in some cases the external variables influencing the process might be not observable. Hence, the Meta-MDP can be modeled by creating a different task representation. However, the gradient itself already implicitly includes information regarding the transition and reward probabilities: what is lost when we do not consider the explicit parametrization of the task? We address this question by retraining our models, and showing the results in Figure \ref{fig:task:notask}: in general, there is no big loss in the performance, especially for Minigolf environment; however, in Meta CartPole, the task parametrization seem to be informative to the choice of the step size. 

\begin{figure}[t]
\centering
	\includegraphics[width=.9\textwidth]{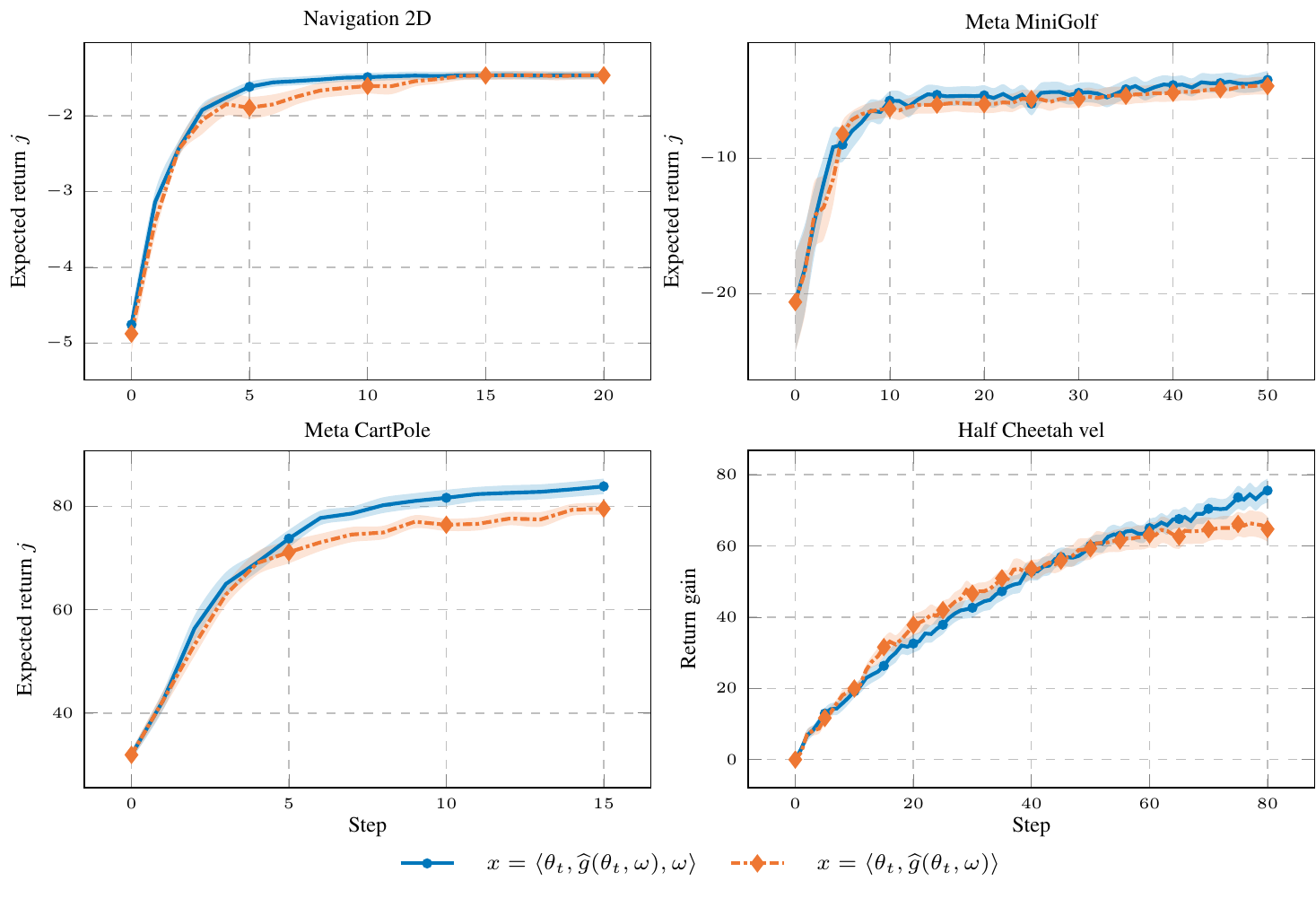}
	\caption{FQI model performance obtained by considering or excluding the explicit task parametrization $\vomega$. (95\% c.i.)}
	\label{fig:task:notask}
\end{figure}

\clearpage
\subsection*{Robustness of FQI regression and effects of Double Q learning.}
In Figure \ref{fig:learning}, we analyzed the results of a FQI model, evaluating the performance under different random policy initializations and tasks. One may wonder if our approach is robust with respect to the randomness included in the ExtraTrees regression. Hence, we trained different FQI models by setting 5 different random states (from 0 to 4), which controls the sampling of the features to apply a split and the draw of the splits. The random state is then equivalent to a a seed for the Extra Trees, and it is applied up to the third FQI iteration. At this point, the models are tested on 20 random task/policy pairs, and their average return gain is taken for each learning step. As we can see in the left plot in Figure \ref{fig:cheetah_seeds}, the $95\%$ confidence interval, that are computed by comparing the different random states, is small enough to claim the robustness of our approach when applied to the Half-Cheetah environment. 

Finally, the right plot in Figure \ref{fig:cheetah_seeds} shows the effects of Clipped Double Q-Learning described in Section \ref{ssec:double}, which is compared with the standard FQI approach with a single Q value function (both are trained with the same number of iterations $N=3$). The latter is still capable of choosing a dynamic learning rate obtaining better results than a fixed step while the former, as expected, provides even better return gains and lower variance over the same set of random test tasks and initial policies.

\begin{figure}[t]
\centering
\includegraphics[width=.9\textwidth]{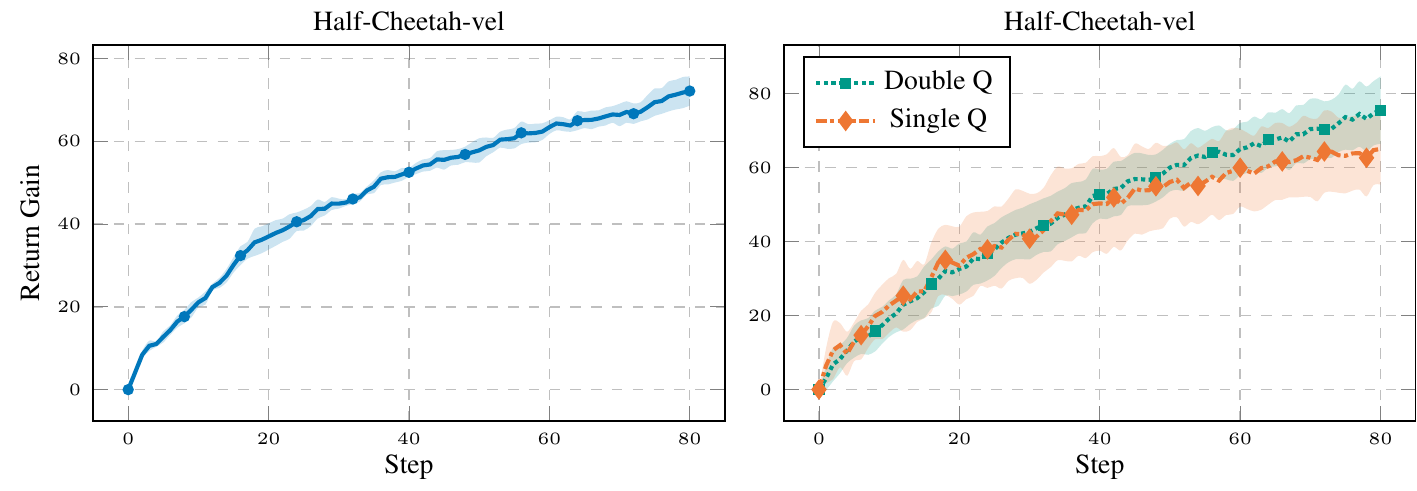}
\caption{On the left: FQI model performance evaluated with 5 different random states (mean $\pm 95\%$ c.i. with respect to the average return gain for each random state). On the right: comparison between the adoption of a single Q and a Double Q-Learning approach for FQI (mean $\pm 95\%$ c.i. with respect to 20 different random test context and initial policies).}
\label{fig:cheetah_seeds}
\end{figure}

\subsection*{Experiments with fixed contexts}
The field of application of metaFQI was presented up to this point as a contextual MDP, with a set of possible tasks. One natural question the reader might wonder is: can we apply this approach to a standard MDP? The answer is trivially positive: we performed some experiments by fixing the task/context and the results are shown in figure \ref{fig:fix_task}. the \textit{metaFQI} curve is related to the performance in test obtained by our agent, and $N$ denotes the iteration selected. Analogously, the NGA curve is related to the best constant learning rate, optimized by means of a grid search. As in the related meta-MDP environments, our approach can outperform the choice of a fixed step size.
\begin{figure}[t]
\centering
\includegraphics[width=.9\textwidth]{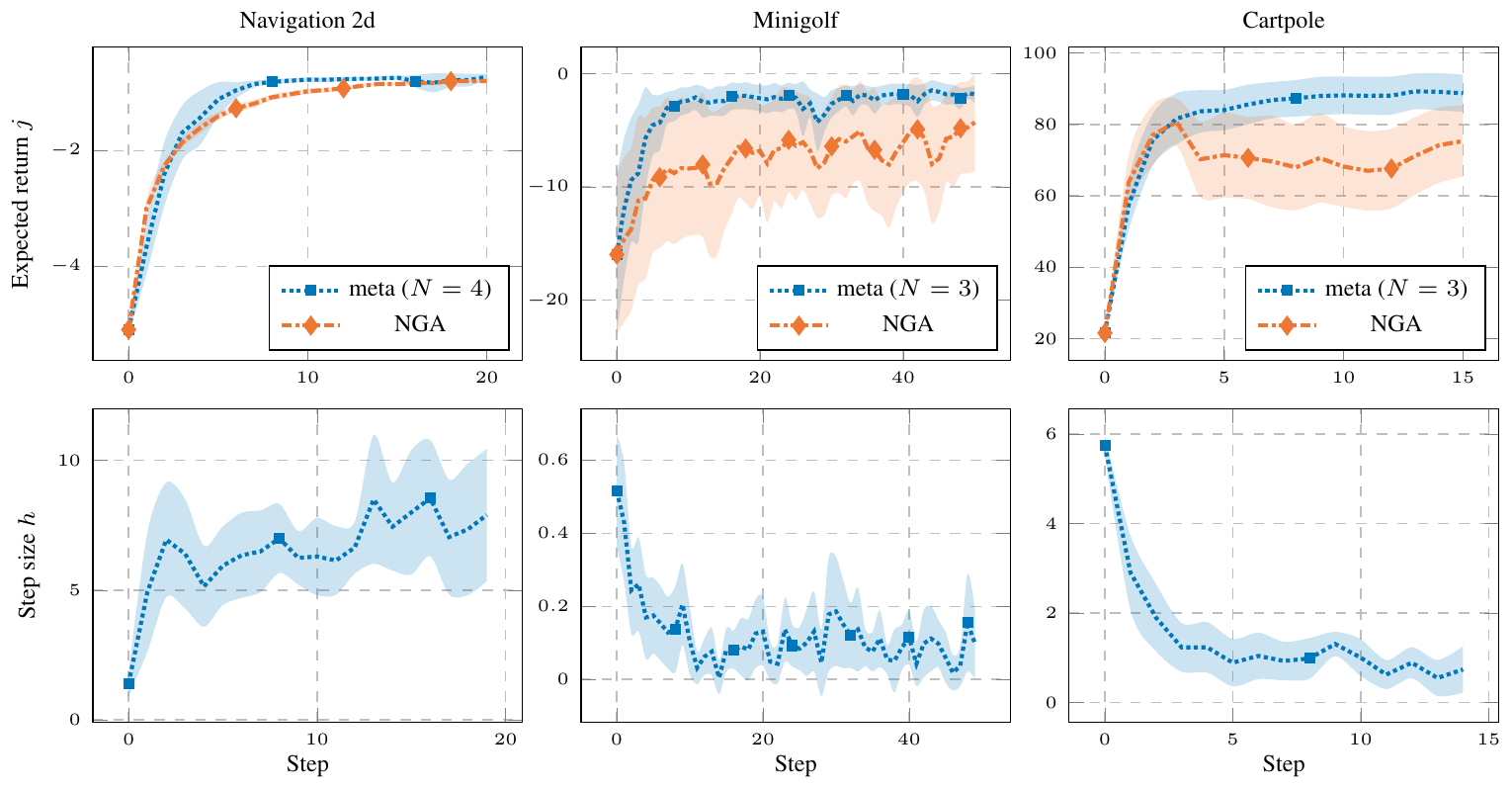}
\caption{FQI model performance and tuned NGA on environments with fix context.}
\label{fig:fix_task}
\end{figure}

\subsection*{Selection of a single learning rate}
One of the most important benefits of the adoption of an adaptive learning rate is provided by an increased number of degrees of freedom w.r.t. the choice of a single learning rate: hence, as an ablation study, we developed an agent capable of choosing only a fixed learning rate in the meta-MDP setting. The learning process has been conducted as follows: at first, we collected a set of trajectories by randomly varying the initial policy and the context, and selecting a random stepsize. Then, we performed a regression (with the same Extra-Trees architecture used for the meta-FQI agent) giving as input the context and the step size $x_i=(\vomega_i, h_i)$, and as output the final performance obtained in the trajectory, said $J_i$. Finally, in the performance evaluation, the sampled test context $\vomega_t$ is given as input to the agent, and the step size selected is then the one that attains the maximum estimated performance $h_t =\argmax_h \widehat{J}_{\vomega_t}(\vtheta_0, h)$. To have a fair comparison, the number of trajectories in the dataset is the same as for the Meta-FQI case. However, this new agent is only interested in the final returns, and not in the whole learning trajectory, hence the amount of data in input is reduced by a factor equal to the learning horizon $H$.
The results are shown in figure \ref{fig:fix_act}, with the curve labeled as \textit{Meta-single-action}. As we can see the agents, even if capable of adapting the stepsize with different contexts, it is still unable to improve the NGA baseline, which has the same step size for each task. This is probably related to the fact that, to provide a fine performance estimation, the model needs a larger amount of samples, while meta-FQI uses all the single steps in the trajectory in the training dataset.

\begin{figure}[t]
\centering
\includegraphics[width=.9\textwidth]{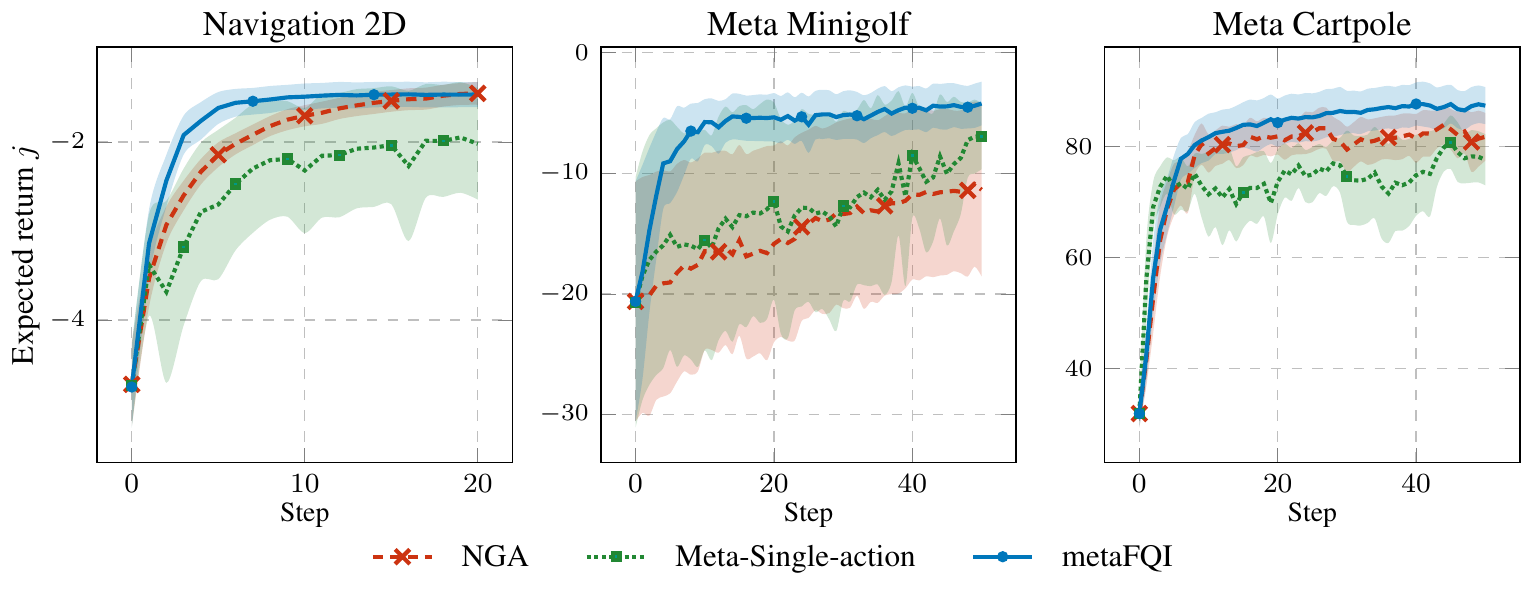}
\caption{Comparison of performances obtained by a trained agent capable of choosing a fixed initial learning rate.}
\label{fig:fix_act}
\end{figure}

\end{document}